\pdfoutput=1
\documentclass[nonacm, sigconf]{acmart}

\PassOptionsToPackage{numbers, sort}{natbib}

\usepackage{amsmath, amssymb, amsthm}
\usepackage{bm}
\usepackage{xcolor, graphicx}
\usepackage{xspace}
\usepackage{subcaption}
\usepackage{url}
\usepackage{enumerate}
\usepackage{sidecap}
\usepackage{standalone}
\usepackage{caption}
\usepackage{multirow}

\newif\ifarxiv\arxivtrue

\newtheorem{definition}{Definition}
\newtheorem{proposition}{Proposition}

\newcommand{\E}{\mathop{\mathbb{E}}}

\newcommand{\R}{\mathbb{R}}

\newcommand{\dn}{\ensuremath{\mathcal{N}}\xspace}
\newcommand{\ds}{\ensuremath{\mathcal{S}}\xspace}
\newcommand{\dx}{\ensuremath{\mathcal{X}}\xspace}

\newcommand{\flipset}{\ensuremath{F(h,G)}\xspace}
\newcommand{\posflip}{\ensuremath{F^+(h,G)}\xspace}
\newcommand{\negflip}{\ensuremath{F^-(h,G)}\xspace}
\newcommand{\posflipp}{\ensuremath{F^+(h,G')}\xspace}
\newcommand{\negflipp}{\ensuremath{F^-(h,G')}\xspace}
\newcommand{\posflipf}{\ensuremath{F^+(h,f)}\xspace}
\newcommand{\negflippf}{\ensuremath{F^-(h,f')}\xspace}
\newcommand{\negflipf}{\ensuremath{F^-(h,f)}\xspace}
\newcommand{\posflippf}{\ensuremath{F^+(h,f')}\xspace}
\newcommand{\starflip}{\ensuremath{F^\star(h,G)}\xspace}

\begin{document}
	\title{FlipTest: Fairness Testing via Optimal Transport}
	
	\author{Emily Black}
	\authornote{The first two authors contributed equally to this research.}
	\affiliation{\institution{Carnegie Mellon University}}
	\email{emilybla@andrew.cmu.edu}
	
	\author{Samuel Yeom}
	\authornotemark[1]
	\affiliation{\institution{Carnegie Mellon University}}
	\email{syeom@cs.cmu.edu}
	
	\author{Matt Fredrikson}
	\affiliation{\institution{Carnegie Mellon University}}
	\email{mfredrik@cs.cmu.edu}
	
	\begin{abstract}
		We present FlipTest, a black-box technique for uncovering discrimination in classifiers.
		FlipTest is motivated by the intuitive question: \emph{had an individual been of a different protected status, would the model have treated them differently?} Rather than relying on causal information to answer this question, FlipTest leverages optimal transport to match individuals in different protected groups, creating similar pairs of in-distribution samples. We show how to use these instances to detect discrimination by constructing a \emph{flipset}: the set of individuals whose classifier output changes post-translation, which corresponds to the set of people who may be harmed because of their group membership. 
		To shed light on \emph{why} the model treats a given subgroup differently, FlipTest produces a \emph{transparency report}: a ranking of features that are most associated with the model's behavior on the flipset.
		Evaluating the approach on three case studies, we show that this provides a computationally inexpensive way to identify subgroups that may be harmed by model discrimination, including in cases where the model satisfies group fairness criteria.
	\end{abstract}
	
	\maketitle
	
	\section{Introduction}
With the recent introduction of machine learning in sensitive applications like predictive policing~\cite{compas} and child welfare~\cite{allegheny}, the question of whether these algorithms can lead to unfair outcomes has gained widespread attention.
These concerns are not merely hypothetical.
Racial bias in the COMPAS recidivism prediction model~\cite{angwin2016machine} and gender bias in Amazon’s hiring model~\cite{dastin2018amazon} suggest that discriminatory models can have wide-reaching harmful effects.

A growing set of strategies have emerged for testing and detection of such discriminatory behaviors.
A common approach that applies to group fairness criteria such as demographic parity~\cite{feldman2015certifying} and equalized odds~\cite{hardt2016equality} is to measure aggregate statistics of the model's behavior on a targeted population. 
For example, this approach was taken with the COMPAS system for recidivism prediction by measuring false positive and negative rates across Caucasian and minority populations~\cite{compas}, and is supported by IBM's AIF360 toolkit for assessing model fairness~\cite{AIF360}.
However, there are several potential issues with this approach~\cite{chouldechova2017fair, woodworth2017learning, kleinberg2017inherent}, among which is that models can potentially ``pass'' such audits while still behaving unfairly towards individuals, or even targeted subgroups~\cite{lipton2018does}.
Additionally, while aggregate statistics can reveal broad patterns of potential discrimination, they do not reveal additional information that sheds light on the underlying discriminatory mechanism at play, which is crucial when assessing whether the behavior is truly problematic.

Recent work ~\cite{galhotra2017fairness, agarwal2018automated} instead searches for discrimination at the individual level, testing whether changes in the protected demographic status of an individual can cause changes in model outcome.
However, to change the protected demographic status of an individual, these methods simply flip the value of the protected attribute (e.g., race or gender).
While this can ensure that the model does not directly use the protected attribute to discriminate, it still allows the model to disproportionately harm a protected group by using features that are correlated with the protected attribute.

The framework of \emph{counterfactual fairness} by Kusner et al.~\cite{kusner2017counterfactual} takes these correlations into account by assuming a causal generative model for the relevant data.
This approach has the advantage that instances of discrimination against individuals or small subgroups cannot ``fly under the radar'', and the causal generative model may lead to a more nuanced and granular understanding of how the model discriminates.
However, the reliance on detailed causal information creates practical issues that may limit its applicability as well.
Namely, it may not be feasible to assume access to a generative causal model in many applications, and if an inaccurate model is used, then the conclusions may be misleading.
Moreover, the legal frameworks governing discrimination in many countries (e.g., \emph{disparate impact} in the US~\cite{griggs1971} and \emph{indirect discrimination} in the UK~\cite{uk-equality}) do not require a causal relationship with the protected status, so tests based on counterfactual fairness may fail to identify instances of legally actionable discrimination.

In this paper we present FlipTest, a \emph{black-box, efficient, and interpretable} fairness testing approach that is motivated by the following intuitive question: \emph{had an individual been of a different protected status, would the model have treated them differently?}
In contrast to aggregate testing methods, FlipTest reasons about the model's behavior on individuals and subgroups to look for evidence of discrimination, and can thus uncover forms of discrimination to which group fairness measures are blind. 
However, unlike counterfactual fairness, FlipTest \emph{does not rely on causal information}, and instead uses an \emph{optimal transport mapping}~\cite{villani2008optimal} to answer the question above. 
Consequently, the goal of our test is not to demonstrate a causal link between the protected attribute and the model's output, but to showcase salient patterns in a model's behavior that may be indicative of discrimination.
Importantly, this means that FlipTest is sensitive to both statistical and causal discrimination, and does not require strong causal assumptions about the data-generating process.
Further, we show that the information computed in this process can provide insight into not just \emph{whether} a model discriminates, but \emph{how} it does and \emph{who} is likely to be affected.

\vspace*{1ex}
\noindent
\textbf{Problem setting.}
We consider a setting where a machine learning system is being audited for discriminatory behaviors, either by well-intentioned stakeholders who may have been involved in the model's construction, or by concerned practitioners outside of the development process.
Ideally, the auditors include a domain expert who is familiar with the application and the subject population who will come into contact with the algorithm. 
We assume that these individuals and those responsible for training the model are not intentionally trying to evade a finding of discrimination.

\vspace*{1ex}
\noindent
\textbf{Optimal transport.}
FlipTest uses an \emph{optimal transport map}~\cite{villani2008optimal} to construct instances that may reveal whether a model's behavior is sensitive to changes in protected status.
An optimal transport map transforms one probability distribution into another, while minimizing a given cost defined over their respective supports.
For example, we might use an optimal transport map from the distribution of men to women in order to obtain a (female, male) pair of inputs with which to query the model.
If the model's output differs for these two people, then it may be evidence that the model discriminates on the basis of gender.
Using optimal transport to compare protected group outcomes is advantageous because it translates exactly from one distribution to another, generating inputs that are in the distribution of its image.
When the image of an optimal map corresponds to a distribution that the model was trained on, the results will reflect characteristic model behavior that can be expected when the model is deployed.
This is not necessarily true for other methods of generating alternate inputs on which to compare model outcomes, e.g. input influence measures~\cite{datta2016algorithmic}.
Further, the mapping does not rely on causal information, and can reveal associative forms of discrimination that causal tests cannot while requiring fewer assumptions about that data.

A key challenge with this approach lies in constructing the mapping, which can be computationally demanding with large, high-dimensional datasets.
In recent years there have been notable advances in methods for efficiently approximating optimal transport maps~\cite{peyre2019computational}, and FlipTest's efficacy can benefit from ongoing work on this problem.
In this paper, we present an approximation method based on generative adversarial networks (GANs)~\cite{goodfellow2014generative} (Section~\ref{sec:gan}), and validate it by showing that it is feasible to construct good, stable approximations of known precise mappings (Section~\ref{sec:validation}).

\vspace*{1ex}
\noindent
\textbf{Finding evidence of discrimination.}
Beyond examining model behavior on individual pairs, we show how the information provided by the optimal transport map can be systematically evaluated for evidence of discrimination.
In particular, we assume that the classifier in question produces binary outputs, one of which is seen as a favorable outcome and the other as unfavorable. 
We consider two sets of individuals under the optimal map: those whose prediction changes from favorable to unfavorable, and vice versa.
We call these \emph{flipsets}, and look to the relative size of each flipset for signs of potential discrimination; for example, we show how flipsets relate to well-known fairness criteria like demographic parity and equalized odds (Section~\ref{sec:fliptr}).
In addition, large flipsets can indicate subgroup-level discrimination that is not well described by these group fairness criteria (Section~\ref{sec:lipton}).
By comparing the distribution of the flipsets to the distribution of the overall population, it is often possible to identify specific subgroups that the model discriminates against.

To gain insight into how the model discriminates, we construct a \textit{transparency report} that summarizes the most salient statistical differences in the features between the flipset individuals and their images under the optimal transport mapping (Section~\ref{sec:tr}). Intuitively, the transparency report can serve as an overview of what features the model may be using to discriminate between populations.

However, it is not guaranteed that a person in the flipset is the victim of discrimination.
For example, an inter-group disparity in the model's output may be permitted if there is a sufficient justification such as a ``business necessity''~\cite[\S II.B]{barocas2016big}.
Therefore, we treat the flipset analysis and transparency report as a starting point for further investigation about the potential unfairness of the model,
which can be followed up with more expensive and conclusive analyses that look at whitebox model information~\cite{datta2016algorithmic, datta2017proxy}.

\vspace*{1ex}
\noindent
\textbf{Experiments.}
We empirically evaluate FlipTest on four different datasets (Section~\ref{sec:application}), demonstrating the testing workflow that we envision for it: Chicago Strategic Subject List~\citep{chicago}, an illustrative dataset from Lipton et al.~\cite{lipton2018does}, the law school success dataset used to illustrate counterfactual fairness~\cite{kusner2017counterfactual}, and a synthetic dataset we construct to demonstrate differences from prior work. 
Our resuls show that FlipTest provides clear, interpretable evidence of discrimination in a range of settings, along with concrete diagnostic information that is useful when reasoning about the model behaviors that are responsible for the discrimination.

We compare FlipTest against two prior approaches: counterfactual fairness~\cite{kusner2017counterfactual} and FairTest~\cite{tramer2017fairtest}.
For counterfactual fairness, we examine the dataset used by the authors to evaluate the approach, and compare FlipTest's optimal transport-based results against those obtained by making comparable interventions on the generative causal model given by Kusner et al.~\cite[Section~5]{kusner2017counterfactual}.
We find that the two lead to similar conclusions about the model's tendency to discriminate, despite the fact that FlipTest makes substantially fewer assumptions about the data.
For FairTest, an approach based on statistical hypothesis testing of subgroup discrimination, we show that FlipTest can complement FairTest by detecting instances that FairTest misses.

\vspace*{1ex}
\noindent
\textbf{Summary.}
To summarize, our main contributions are: \emph{(1)} FlipTest, a black-box, efficient testing approach for detecting discrimination in classifiers; \emph{(2)} the novel use of optimal transport for fairness testing; and \emph{(3)} the application of FlipTest to two case studies involving predictive policing (Section~\ref{sec:ssl}) and hiring (Section~\ref{sec:lipton}), as well as comparisons to prior fairness testing methods (Sections~\ref{sec:counterfactual}, \ref{sec:fairtest}), which demonstrate that our approach can identify concrete examples of unfair model behavior in cases where prior testing methods do not.

	\section{An Illustrative Example}
	\label{sec:example}
In this section, we illustrate the main concepts behind FlipTest with a running example, which uses a synthetic dataset created by Lipton et al.~\citep[\S 4.1]{lipton2018does}.
This dataset consists of two features, hair length and work experience, and supposes a binary classifier that uses these features to decide whether a given person should be hired.
We investigate possible gender bias in this model, asking whether the model's output would have been different had a given person been of a different gender.
However, it is not sufficient to simply flip the gender attribute due to correlations in the data that the model may be use as a proxy for gender: in this data, gender is correlated with hair length and work experience.
Additionally, flipping the gender attribute is not an option because the model does not directly use this attribute.

We instead map the set of women in the data to their male correspondents, and analyze cases where the model treats women differently from the men that they are mapped to.
This raises the question of which specific man a given woman should be mapped to, for which we appeal to the intuition that a difference in treatment between two people is not by itself strong evidence of discrimination unless they are similar enough that the disparity cannot be justified.
For example, when a man with 20 years of relevant experience is hired over a woman with no experience, this difference would likely be attributed to work experience rather than gender discrimination.
This motivates our use of an \emph{optimal transport mapping}~\cite{villani2008optimal}, which minimizes the sum of the distances between a woman and the man that she is mapped to (her \emph{counterpart}), where the distance quantifies how different a pair of people are.

We must now specify a distance function (or \emph{cost function}) to operationalize the optimal transport mapping.
Although there are no easy answers to the question of which people are similar for the purpose of establishing discrimination, the goal of this paper is to demonstrate a new technique for finding evidence of possible discrimination rather than to present a conclusive definition of discrimination.
We find that the square of the $L_1$ distance leads to reliable results (Section~\ref{sec:application}), so we use it in this paper, but our approach is compatible with any cost function deemed suitable for the setting.

\begin{figure}
	\centering
	\resizebox{0.8\columnwidth}{!}{\includegraphics{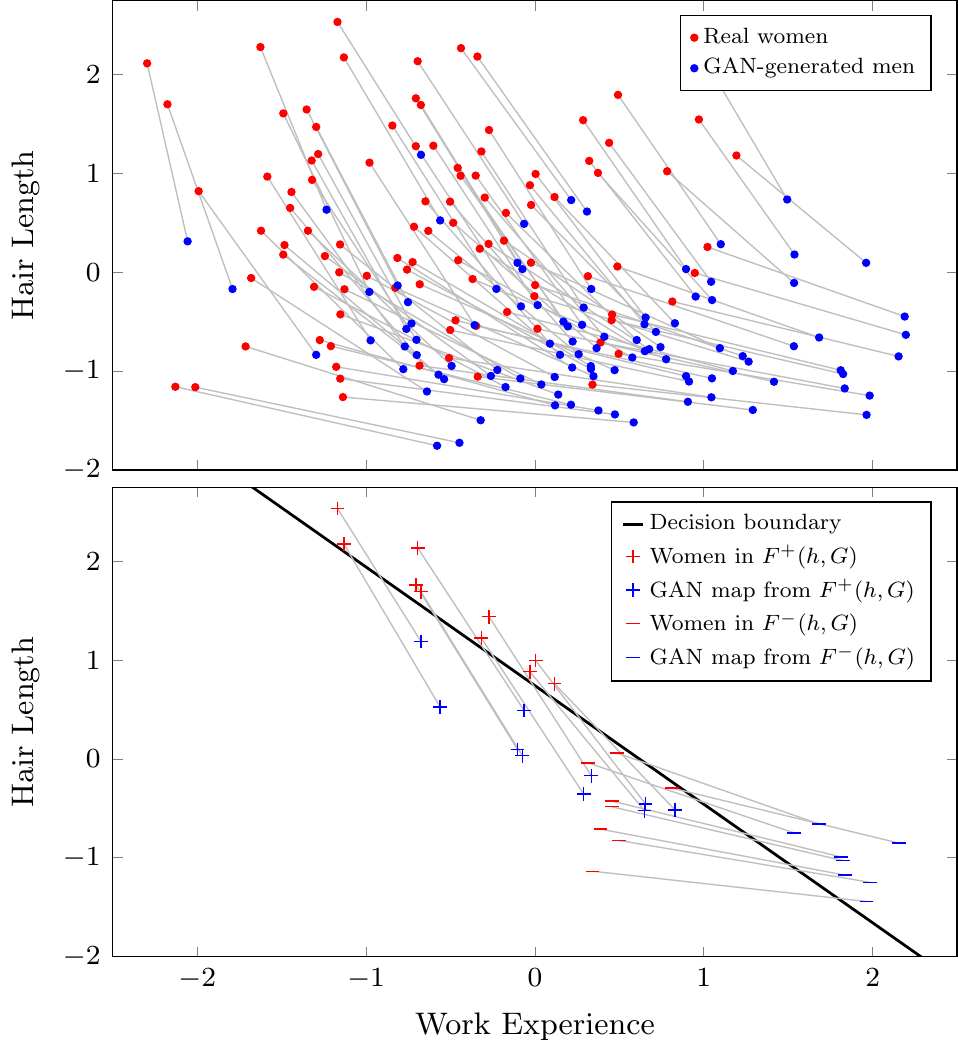}}
	\caption{
		(Top) Optimal transport mapping from women to men in the Lipton et al.~\cite{lipton2018does} synthetic dataset.
		This is an approximation generated by a GAN, as described in Section~\ref{sec:gan}.
		(Bottom) Flipsets as defined by the model with the given decision boundary.
	}
	\label{fig:lipton-scatter}
\end{figure}
We then analyze the optimal transport mapping, which is depicted in the top plot of Figure~\ref{fig:lipton-scatter}.
The analysis is through the \emph{flipset} \flipset, which consists of all women whose outcomes were different from their counterparts'.
We partition the flipset into the \emph{positive flipset} \posflip, which contains the hired women whose counterparts were not, and the \emph{negative flipset} \negflip, which is the set of rejected women whose counterparts were hired.
Thus, in some sense the women in \posflip were advantaged due to their gender, and those in \negflip disadvantaged.
Although this is not sufficient to establish that gender \emph{caused} the difference in the way that some causal tests~\cite{kusner2017counterfactual,datta2016algorithmic} can, FlipTest has the advantage that it queries the model on in-distribution points only, so its response to these inputs is likely to be reliable.

We begin by examining the size of the positive and negative flipsets.
Suppose that the model does not satisfy demographic parity, hiring disproportionately more men than women.
Then, the flipsets will have different sizes, with the negative flipset larger than the positive.
Therefore, a large difference in the sizes of the flipsets is evidence of possibly discriminatory behavior in the model.
However, such differences may also be based on a justifiable reason, such as when the job in question has requirements that are more likely to be satisfied by a particular gender.

Alternatively, it could be that the positive and negative flipsets have the same size.
If this is the case, i.e., the \emph{net} flipset size is zero, then the model satisfies demographic parity.
However, if the sizes of the individual flipsets are still large, then there may be discrimination at the subgroup level.
To investigate which subgroups may be discriminated against (or unfairly advantaged), we can compare the distributions of the flipsets to that of the entire population.
We plot the marginal distributions in Figure~\ref{fig:lipton}, and the results show that the advantaged (\posflip) women tend to have much longer hair than the disadvantaged (\negflip) women, suggesting that the model may be discriminating against shorter-haired women.

Finally, we can produce a \emph{transparency report} to gain more information about why the model may behave in this way.
The transparency report describes how the members of the flipsets are different from their counterparts, shedding light on which features may have contributed to the model's decision to classify the counterparts differently.
As we can see in the bottom plot of Figure~\ref{fig:lipton-scatter}, the women in the negative flipset have much less work experience than their counterparts, and this suggests that they were not hired because of inadequate work experience.
Although this method is not foolproof because work experience could have been a correlate of another feature that the model actually uses, it points to a specific aspect of model behavior for further investigation with tools such as QII~\cite{datta2016algorithmic} that can ascertain which feature is most responsible for the model's behavior.
If it turns out that work experience causes the difference in the hiring decisions---and in our example it does---a practitioner can consult a domain expert to decide whether the use of work experience is justified.
In many cases it would be justified, but it may not if the disparity in work experience is due to discrimination.

In the rest of this paper, we solve the main technical challenges behind FlipTest and experimentally verify that it reaches the correct conclusions in many settings.
We formally describe the optimal transport mapping in Section~\ref{sec:transport}, and we show how to use a GAN to efficiently approximate an optimal transport mapping in a way that generalizes to unseen data.
While the example in this section is based on demographic parity, we can split the dataset by the true label to test for a more individualized notion of equalized odds as well.
Our experiments in Sections~\ref{sec:ssl} and \ref{sec:lipton} show that FlipTest can identify possible subgroup discrimination in cases where the relevant group fairness objective (demographic parity or equalized odds) is met and in cases where it is not.
Finally, in Sections~\ref{sec:counterfactual} and \ref{sec:fairtest} we compare FlipTest to prior similar methods.

	\section{Optimal Transport Mapping}
	\label{sec:transport}

In this section, we describe the optimal transport problem in more detail, and show in Section~\ref{sec:gan} how to solve a GAN objective to approximate the optimal transport mapping in a way that generalizes to unseen data points.
In Section~\ref{sec:stability} we compare this GAN approximation to the exact optimal transport mapping and another approximation method by Seguy et al.~\cite{seguy2017large}, finding that the GAN tends to give more stable mappings.
Although this GAN approximation is not the only way to operationalize FlipTest, we use the GAN approximation throughout the rest of this paper because it appeared to give more reliable results than the alternatives we considered.

We first introduce the notation.
Let $\ds$ and $\ds'$ be two distributions defined over the feature space $\dx$.
In practice, we do not know these distributions, so we usually deal with observations of points drawn from these distributions instead.
We will use the sets $S = \{\boldsymbol{x}_1, \ldots, \boldsymbol{x}_n\}$ and $S' = \{\boldsymbol{x}_1', \ldots, \boldsymbol{x}_n'\}$ to denote the observed points, where $n = |S| = |S'|$. Note that here we assume that $|S| = |S'|$ for ease of exposition, as this assumption allows the resulting exact optimal transport mapping to be deterministic. The general case where $|S| \neq |S'|$ can be handled through the use of randomized optimal transport mappings, and our approximation methods do not require that the two sets have equal size.

Let $c: \dx \times \dx \to [0, \infty)$ be a cost function that describes the cost of moving between two points in the feature space $\dx$.
Intuitively, an optimal transport mapping from $S$ to $S'$ is a minimum-cost way to move the points in $S$ such that the end result is $S'$.
Thus, if more similar pairs of points have a lower cost, an optimal transport mapping describes how to match points in $S$ with their similar counterparts in $S'$.
Formally, an optimal transport mapping can be defined as a bijection $f: S \to S'$ that minimizes the expected cost $\E[c(\boldsymbol{x}, f(\boldsymbol{x}))] = \frac{1}{n} \sum_{i=1}^n c(\boldsymbol{x}_i, f(\boldsymbol{x}_i))$.

\subsection{Approximation via GANs}
\label{sec:gan}
While the exact optimal transport mapping between $S$ and $S'$ can be solved through a linear program or the Hungarian algorithm~\cite{kuhn1955hungarian}, these methods do not scale well to large $n$.
In addition, the exact mapping, as well as some approximations thereof~\cite{altschuler2017near, quanrud2018approximating}, is not defined for points outside of $S$.
Therefore, we instead propose a generative adversarial network (GAN)~\cite{goodfellow2014generative} to approximate an optimal transport mapping in a way that avoids both of these issues.

Because we want to use the generator $G$ as an optimal transport mapping, we assume that its inputs draw randomness from $\ds$.
For concreteness, we base our construction on the Wasserstein GAN~\cite{arjovsky2017wasserstein}, and note that our primary result (Proposition~\ref{thm:gan} below) can be extended to other types of GANs as well.
When training a conventional Wasserstein GAN with the sets of observed points $S$ and $S'$, the generator's loss function is $(1/n)\sum_{\boldsymbol{x} \in S} D(G(\boldsymbol{x}))$ for discriminator $D$, and the discriminator's loss function is $(1/n)\sum_{\boldsymbol{x}' \in S'} D(\boldsymbol{x}') - (1/n)\sum_{\boldsymbol{x} \in S} D(G(\boldsymbol{x}))$.
For the purpose of finding an optimal transport mapping, we modify the generator's loss function to take into account the cost of moving from a point in $S$ to a point in $S'$:
\begin{equation}
L_G = \textstyle \frac{1}{n} \sum_{\boldsymbol{x} \in S} D(G(\boldsymbol{x})) + \frac{\lambda}{n} \sum_{\boldsymbol{x} \in S} c(\boldsymbol{x}, G(\boldsymbol{x})) \label{eqn:loss}
\end{equation}

Our modified generator has two objectives, with the parameter $\lambda$ controlling their relative importance: generating the correct output distribution $\ds'$, and minimizing the expected cost $c(\boldsymbol{x}, G(\boldsymbol{x}))$.
Proposition~\ref{thm:gan} formalizes the intuition that these objectives are also those of an optimal transport mapping.
The proof is given in
\ifarxiv
Appendix~\ref{app:proof}.
\else
Appendix~\ref{app:proof} of the supplementary material.
\fi
\begin{proposition}
\label{thm:gan}
Suppose that $G^*$ is a minimizer of $L_G$ among all $G$ such that $G(S) = S'$.
If $\lambda > 0$, $G^*$ is an exact optimal transport mapping from $S$ to $S'$.
\end{proposition}
Although the generator $G$ will not satisfy $G(S) = S'$ in practice, Proposition~\ref{thm:gan} motivates the use of this generator to approximate an optimal transport mapping.
Our experimental results (Section~\ref{sec:stability}) show that the approximate GAN mapping is more stable than the exact mapping, which is not very stable and changes drastically depending on which sets $S$ and $S'$ were drawn from $\ds$ and $\ds'$.

\subsection{Stability}
\label{sec:stability}
\begin{figure}
\resizebox{0.8\columnwidth}{!}{\includegraphics{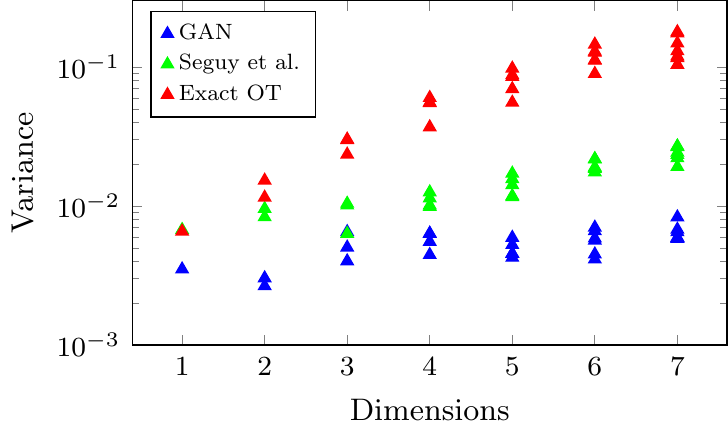}}
\caption{Variance of the GAN (blue), Seguy et al.~\cite{seguy2017large} (green), and exact (red) optimal transport mappings over the different random draws of the observed points $S$ and $S'$.
Note the logarithmic vertical scale.
The horizontal axis represents the number of dimensions in the feature space $\dx$, and each plotted point represents the variance of one of the features.
More details are given in Section~\ref{sec:stability}.}
\label{fig:stability}
\end{figure}
Here, we compare the behavior of the exact optimal transport mapping to those of approximate mappings.
This is not intended to be a comprehensive evaluation of all approximation methods, but rather an argument for the use of approximations over the exact mapping for our purpose.
We note that FlipTest is compatible with any optimal transport method.

To measure stability, we fix a point $\boldsymbol{x} \in S$ and then draw multiple distinct samples of the other $n-1$ points from $\ds$ and $n$ points from $\ds'$.
Thus, we have different sampled sets $S$ and $S'$ each time, and we observe the variance of the point $f(\boldsymbol{x})$ over the random draws.

As with all experiments in this paper, we used the square of the $L_1$ distance as the cost function.
The exact optimal transport mapping was computed as a linear program with Gurobi~\cite{gurobi} implemented in Python~3, and for the GAN approximation (Section~\ref{sec:gan}), we trained a Wasserstein GAN~\cite{arjovsky2017wasserstein} using Keras~\cite{keras} with the TensorFlow backend~\cite{tensorflow}.
For these experiments, we mapped the standard multivariate normal distribution to itself.
Because the size of the linear program for the exact optimal transport mapping increases at least quadratically with the size $n$ of the dataset, we used $n = 500$. 
Each experiment was repeated with 100 
random draws of the dataset.

In the first set of experiments, we set $\boldsymbol{x}$ to be the one vector and observed the mean of $f(\boldsymbol{x})$.
Since we map a distribution to itself, $f$ should roughly be the identity function, and the mean of $f(\boldsymbol{x})$ should be similar to $\boldsymbol{x}$.
While this was the case for the GAN approximation, the exact mapping displayed a significant ``regression-to-the-mean'' effect that increased with the number of dimensions in the feature space.

In the second set of experiments, we set $\boldsymbol{x}$ to the zero vector and noted the variance of $f(\boldsymbol{x})$ under all three types of mappings.
The result is plotted in Figure~\ref{fig:stability}, showing that the GAN approximation is much more stable than the exact mapping and somewhat more stable than that obtained by the method described by Seguy et al.~\cite{seguy2017large}.

The differences in both mean and variance persisted when we changed the data distribution by making the features correlated with each other.
These differences can be explained by the fact that the exact mapping tends to overfit to the observed points, since it has to map every point to another observed point.
As a result, approximate mappings are better suited for evaluating the fairness of a model that is trained to generalize.
Since the GAN mapping appears to be more stable than that of Seguy et al., for the rest of the experiments we will exclusively use GANs as the optimal transport method in FlipTest.
At the time of writing, we have not been able to evaluate the very recent GAN-based optimal transport approximation method by Leygonie et al.~\cite{leygonie2019adversarial} for use in FlipTest, but any advantage of their method over the construction given here would translate to an improvement for FlipTest's results.

	\section{Flipsets and Transparency Reports}
	\label{sec:fliptr}
We leverage the optimal transport mapping to gather two main pieces of information from a model: who may experience discrimination, and which features may be associated with this effect.
In Section~\ref{sec:flipset}, we describe \emph{flipsets}, which we use to answer the first question, and in Section~\ref{sec:tr} we show how to use them to construct \emph{transparency reports}, which help answer the second question.

\subsection{Flipsets}
\label{sec:flipset}
We begin by introducing the \emph{flipset} (Definition~\ref{def:flipset}), which is the set of points whose image under a transport mapping is assigned a different label by a binary classifier.

\begin{definition}[Flipset]
\label{def:flipset}
Let $h: \dx \to \{0,1\}$ be a classifier and $G: \ds \to \ds'$ be an optimal transport mapping (or an approximation).
The flipset \flipset is the set of points in $S$ whose mapping into $\ds'$ under $G$ changes classification.
\begin{equation}
\label{eqn:flipset}
\flipset = \{\boldsymbol{x} \in S \mid h(\boldsymbol{x}) \ne h(G(\boldsymbol{x}))\}
\end{equation}
The \emph{positive} and \emph{negative} partitions of \flipset are denoted by \posflip and \negflip.
\begin{align*}
\posflip &= \{\boldsymbol{x} \in S \mid h(\boldsymbol{x}) > h(G(\boldsymbol{x}))\} \\
\negflip &= \{\boldsymbol{x} \in S \mid h(\boldsymbol{x}) < h(G(\boldsymbol{x}))\}
\end{align*}
\end{definition}
In our experiments, $\ds$ and $\ds'$ will correspond to two groups with differing values for a protected attribute, and $h$ will be a classifier with the potential to be unfair.
For example, suppose that $\ds$ and $\ds'$ respectively correspond to female and male job applicants and that $h$ is used to determine which applicants should proceed to further rounds of interview.
Then \posflip is the set of female applicants who proceed to the next round but whose male counterparts under $G$ do not, and \negflip is the women who do not proceed but whose male counterparts do.

Note that we can also create flipsets based on a mapping $G' : \ds' \to \ds$ in the opposite direction.
Then, in our example \posflipp is the set of male applicants who proceed to the next round but whose female counterparts under $G'$ do not, and \negflipp is the men who do not proceed but whose female counterparts do.
If $G'$ and $G$ compose to the identity function, these flipsets would have the same size as \negflip and \posflip, respectively, and we can test for this property as a sanity check of our GAN mappings.
We expand this discussion in
\ifarxiv
Appendix~\ref{app:backwards}.
\else
Appendix~\ref{app:backwards} of the supplementary material.
\fi

If the distributions $\ds$ and $\ds'$ are equal, we would expect $G$ to be the identity function, leading to empty positive and negative flipsets.
This corresponds to the setting where the input features are independent of the protected attribute, thereby ensuring that the model cannot discriminate on the basis of the protected attribute.
Proposition~\ref{thm:demparity} shows that demographic parity only provides a weaker guarantee of zero \emph{net} flipset size.
Thus, if the positive and negative flipsets are nonempty but have equal size, some individuals may be experiencing discrimination even though demographic parity holds.
In this proposition, we use the exact optimal transport mapping to avoid any noise introduced by the GAN approximation, and the proof is provided in
\ifarxiv
Appendix~\ref{app:proof}.
\else
Appendix~\ref{app:proof} of the supplementary material.
\fi

\begin{proposition} \label{thm:demparity}
Let $h$ be a binary classifier and $G: S \to S'$, with $|S| = |S'|$, be the exact optimal transport mapping.
Then, $|\posflip| = |\negflip|$ if and only if the model satisfies demographic parity on the observed points, i.e.,
\[ |\{\boldsymbol{x} \in S \mid h(\boldsymbol{x}) = 1 \}| = |\{\boldsymbol{x}' \in S' \mid h(\boldsymbol{x}') = 1 \}|. \]
\end{proposition}

If we instead consider the distributions $\ds|(y=1)$ and $\ds'|(y=1)$, conditioned by the true label $y$, we can prove a similar result about equality of opportunity~\cite{hardt2016equality}, and if we consider $\ds|(y=0)$ and $\ds'|(y=0)$ as well, we can extend the result to equalized odds~\cite{hardt2016equality}.

When $h$ is biased, the flipsets can provide several additional forms useful information about the model's behavior.
First, the relative sizes of \posflip and \negflip can serve as a simple test of group fairness.
Second, the absolute sizes of \posflip and \negflip, if they are large, can indicate possible discrimination at the subgroup level.
Third, if the distributions of the flipsets are different from \ds, we gain information about \emph{which} subgroup may be discriminated against.
We illustrate these insights in greater detail in the case studies described in Section~\ref{sec:application}.

\subsection{Transparency Reports}
\label{sec:tr}

A \emph{transparency report} (Definition~\ref{def:transparency}) identifies features that change the most, and most consistently, under $G$ between members of a given flipset (\posflip or \negflip).
These features are likely candidates for the underlying reasons for the observed discrimination, and can be examined further using more costly causal influence methods~\cite{datta2016algorithmic} to make a final determination.

\begin{definition}[Transparency Report]
\label{def:transparency}
Let $h: \dx \to \{0,1\}$ be a classifier, $G: \ds \to \ds'$ be an optimal transport mapping (or approximation), and \flipset be the corresponding flipset.
If $\dx \subseteq \R^d$, we can compute the following vectors, each of whose coordinate corresponds to a feature in $\dx$:
\begin{align*}
& \textstyle \frac{1}{|\starflip|} \sum_{\boldsymbol{x} \in \starflip} \boldsymbol{x} - G(\boldsymbol{x}), \text{ and} \\
& \textstyle \frac{1}{|\starflip|} \sum_{\boldsymbol{x} \in \starflip} \mathrm{sign}(\boldsymbol{x} - G(\boldsymbol{x}))
\end{align*}
Here, $\star \in \{+, -\}$.
Together, these vectors define a \emph{transparency report}, which consists of two rankings of the features in $\dx$, each sorted by the absolute value of each coordinate.
\end{definition}

Intuitively, the features ranked highest by the transparency report are those that are most associated with the model's differences in behavior on the flipset.
As we show in Sections~\ref{sec:ssl} and \ref{sec:lipton}, these often align closely in practice with the features used by the model to discriminate.

	\section{Experiments}
	\label{sec:application}

We now apply FlipTest to real and synthetic datasets, illustrating its use in finding discrimination in models.
We begin by providing additional validation of the GAN optimal transport approximation (Section~\ref{sec:validation}), and then move on to two case studies: a biased predictive policing model (Section~\ref{sec:ssl}), a well as a hiring model (Section~\ref{sec:lipton}) that contains a subtle form of subgroup discrimination.
In Sections~\ref{sec:counterfactual} and \ref{sec:fairtest}, we compare FlipTest to prior fairness testing methods, namely counterfactual fairness-based auditing~\cite{kusner2017counterfactual} and FairTest~\cite{tramer2017fairtest}.
In all of the experiments, we trained GANs using the configuration described in Section~\ref{sec:transport}.

\subsection{GAN Validation}
\label{sec:validation}
\begin{table}
\caption{Validations for GANs included in the experiments section. KS refers to the Kolmogorov--Smirnov two-sample test statistic. MSE Diff refers to the difference, between real and generated data, of the mean squared error of a linear regression model trained on the real data to predict a feature from the rest. For KS and MSE, we ran 10 trials; the mean is given first, with standard deviation in parentheses. Dist: OT, GAN refers to the average of the squared $L_1$ distance between a data point $x$ and its counterpart $G(x)$ under an exact optimal transport mapping (OT) or an approximated GAN mapping (GAN).}
\label{tbl:validations}
\centering
\resizebox{\columnwidth}{!}{%
\begin{tabular}{llrrr}
Experiment &  Features & KS (std) & MSE Diff (std) & Dist: OT, GAN\\\hline
\multirow{3}{*}{\shortstack[l]{SSL: \\ Dem Parity}}
           & Age       & $0.054$ $(0.001)$ &  $0.070$ $(0.048)$ & \multirow{3}{*}{$2.04$, $1.64$} \\
           & Gang Aff  & $0.018$ $(0.006)$ &  $0.094$ $(0.034)$ & \\
           & Narc Arr  & $0.025$ $(0.002)$ &  $0.298$ $(0.058)$  &\\ \hline
\multirow{3}{*}{\shortstack[l]{SSL: \\ Eq Odds (Neg)}}
           & Age       & $0.019$ $(0.021)$ &  $0.019$ $(0.057)$ & \multirow{3}{*}{$1.52$, $1.42$} \\
           & Gang Aff  & $0.007$ $(0.004)$ & $-0.009$ $(0.053)$ & \\
           & Narc Arr  & $0.019$ $(0.007)$ &  $0.094$ $(0.078)$ & \\ \hline
\multirow{2}{*}{\shortstack[l]{Lipton: \\ Dem Parity}}
           & Work Exp  & $0.072$ $(0.014)$ &  $0.150$ $(0.204)$ & \multirow{2}{*}{$2.19$, $2.09$} \\
           & Hair Len  & $0.074$ $(0.043)$ &  $0.141$ $(0.114)$ \\ \hline
\multirow{2}{*}{\shortstack[l]{Law School}}
           & LSAT      & $0.057$ $(0.012)$ &  $3.757$ $(0.462)$ & \multirow{2}{*}{$10.99$, $11.10$} \\
           & GPA       & $0.110$ $(0.027)$ & $-0.040$ $(0.005)$ &
\end{tabular}
}
\end{table}

In general, the GAN does not converge to the target distribution.
Moreover, limitations in the amount of data available to train the GAN will reduce the accuracy of the approximation.
To evaluate the effect of these factors on flipsets, we trained a GAN with samples from two identical distributions.
In this setting, as the sample size approaches infinity, the exact optimal transport mapping will lead to empty flipsets.
Therefore, the results of this experiment would indicate how many flips we can expect due to the noise in the GAN approximation and the finite sample size.

Because we map a distribution to itself, in order to simulate a more typical application of GAN training, we added additional random features that are \emph{dependent} on the protected attribute.
We drew 10,000 points drawn from each distribution, and to amplify any errors in the GAN mapping, we trained a very complex model by fitting an SVM with RBF kernel to random labels.
Further details on the experimental procedure are given in
\ifarxiv
Appendix~\ref{app:control}.
\else
Appendix~\ref{app:control} of the supplementary material.
\fi
In the end, the flipsets were small despite the above steps that were intended to increase the size of the flipsets: we observed $|\posflip| = 167$ and $|\negflip| = 148$, each accounting for approximately 3\% of the data.
This serves a benchmark for comparison with the models in the following sections, which have larger and more unbalanced flipsets.
This difference is striking given that we would expect larger flipsets from the complicated SVM classifier here than the simpler models in the later experiments.

To further validate our GAN mappings, we computed the exact optimal transport mapping $f$ on a 2,000-member subset of the data and computed the average squared $L_1$ distance between $x$ and $f(x)$. Then, we compared this quantity to the average squared $L_1$ distance between $x$ and $G(x)$ for the GAN generator $G$. If the GAN mapping closely matches the optimal transport mapping, we would expect these numbers to be similar.

We also examined the fit of the GAN-based approximation on the datasets used in the evaluation using the Kolmogorov--Smirnov (KS) two-sample test on the marginal distributions for each feature between the real target data $S'$ and the generated data $G(S)$.
The output of this test corresponds to the largest difference in the empirical distribution functions of the two samples, with a small KS-statistic corresponding to the case where the two distributions are similar.
In Table~\ref{tbl:validations}, we only report the KS-statistic and its variance, and not the associated p-value, since we do not require or expect the two samples to be from the exact same distribution.
We instead use the statistic as a metric to judge how far apart the distributions are, aiming for them to be as close as possible.

Since similarity tests on the marginals of a distribution do not account for correlation between features, we further validate the output by training a linear regression model to predict each feature from the others (e.g., in the SSL dataset, age from the other seven features).
We train these regression models on the real target data $S'$, and compare the accuracies of these models on $S'$ to those on the generated data $G(S)$.
If the mean squared error is similar between predicting all features from the true data and the generated data, we take this as evidence that the GAN has captured correlation between the features well.
For each of the experiments in the following sections, the validation results are also reported in Table~\ref{tbl:validations}. The KS and MSE statistics are reported over 10 trials.



\subsection{Testing a Biased Model}
\label{sec:ssl}

The Chicago Strategic Subject List (SSL) dataset~\cite{chicago} consists of arrest data collected in Chicago for the purpose of identifying which individuals are likely to be involved in a violent crime, either as a victim or a perpetrator. 
We used the following eight features that are also used by the SSL model: number of times as victim of a shooting incident, age during last arrest, number of times as victim of aggravated battery or assault, number of prior arrests for violent offenses, gang affiliation, number of prior narcotic arrests, trend in recent criminal activity and number of prior unlawful use of weapon arrests.
The target feature corresponds to the risk of being involved in a shooting, ranging in value from 0 to 500 (low to high risk).
We normalized all input features to zero mean and unit variance, and a standard least-squares regression model for this data primarily relies on age ($w_{\it age} \approx -50$), giving the remaining features coefficients of magnitude less than 10.

We create a classification task from this dataset by setting a threshold on score (345); 10\% of the dataset has a score above this threshold.
As prior work~\cite{yeom2018hunting} has shown that models trained on this data do not appear to exhibit significant bias, we deliberately bias the classifier by making it rely more heavily on the number of prior narcotics arrests, which is correlated with race ($r = 0.12$) and is arguably less predictive of involvement in violent crime than, say, number of prior arrests for violent offenses.
The resulting linear classifier exclusively uses age and narcotics arrests, classifying a person as high risk if and only if $-53 \cdot \mathit{age} + 25 \cdot \mathit{narc} > 65$.
As a result, the model is expected to assign black subjects higher average scores without necessarily being accurate.

We perform two tests on this model:
in Section~\ref{sec:ssl_dem}, we map between all black and white subjects, testing for criteria related to demographic parity, and in Section~\ref{sec:ssl_eq}, we construct one map between the ground-truth positive (high-risk) black and white subjects, and another map between the ground-truth negative (low-risk) subjects, testing for criteria related to equalized odds.

\subsubsection{FlipTest Demographic Parity}
\label{sec:ssl_dem}

\begin{figure}
\centering
\resizebox{\columnwidth}{!}{
\includegraphics{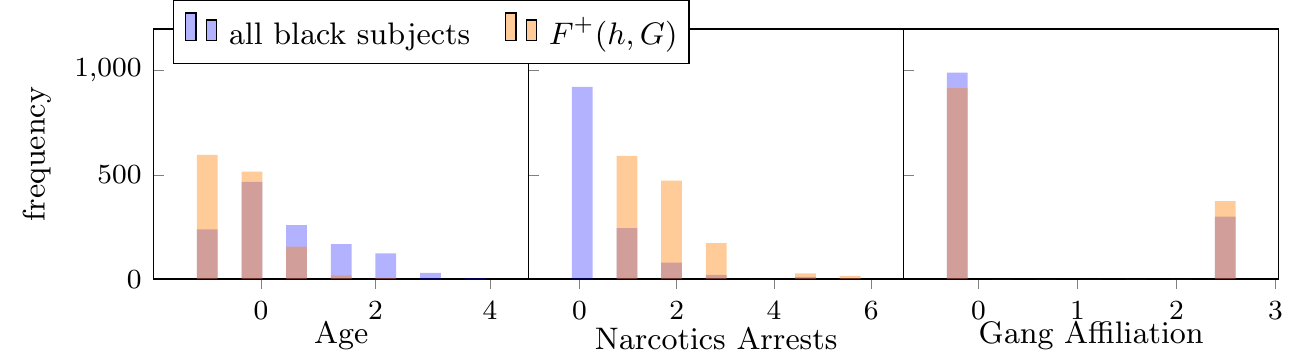}}
\caption{Distribution of the positive flipset for a black-to-white mapping (orange) and the overall black subject population (purple). On the x-axis we have (normalized) feature values; y-axis is number of individuals. We did not plot the flipset in the other direction, as it only has 4 individuals.}
\label{fig:SSL1}
\end{figure}

Additional details on experimental setup are found in
\ifarxiv
Appendix~\ref{app:experimental-setup}.
\else
Appendix~\ref{app:experimental-setup} of the supplementary material.
\fi

The positive flipset \posflip shows that 1,290 black subjects that are marked by the model as high risk are marked as low risk when sent through the black-to-white mapping (out of 3,683 black subjects marked as high risk). 
On the other hand, the negative flipset \negflip consists of only 4 people, which is an extremely small fraction of the 37,877 black subjects marked as low risk.
The size of \posflip and the significant asymmetry of the flipsets suggest that the model discriminates on the basis of race,
which is consistent with how the model was constructed.

To gain more insight about which subgroups may be discriminated against, we investigate the distributions of the flipsets as compared to that of the general black population.
The histograms of the marginal distributions of select features is given in Figure~\ref{fig:SSL1}, with the rest of the features plotted in Figure~\ref{fig:SSL1.2} in the
\ifarxiv
appendix.
\else
supplementary material.
\fi
For age and narcotic arrests, the flipset subjects skew away from the full population, towards younger people with more narcotic arrests.
This suggests the bias of the model affects younger people with more narcotic arrests most.
On the other hand, the marginals largely overlap for gang affiliation, which the model does not directly use.

We then look at the transparency report (Figure~\ref{fig:SSL1.2} in the
\ifarxiv
appendix)
\else
supplementary material)
\fi
to learn which features may be responsible for this apparently biased behavior.
The black subjects in \posflip changed the most, and most consistently, in narcotics arrests under the GAN mapping, which is again consistent with the type of bias that we introduced into this model.
Thus, although the transparency reports only provide direct insight into statistical correlations between the features and model behavior, in this case they identified precisely the feature that is responsible for the model's bias.
In general, based on the information from the transparency report, an practitioner testing a model can decide whether they would like to investigate the model further for bias based on the features in the transparency report.
A more definitive evidence, such as that given by QII~\cite{datta2016algorithmic}, would be required to conclude that the feature causes the model's discriminatory behavior.

In Appendix~\ref{app:ssl} and
\ifarxiv
Figure~\ref{fig:SSL2},
\else
Figure~\ref{fig:SSL2} in the supplementary material,
\fi
we present an additional set of experiments on this data using demographic parity, using a model that discriminates against black subjects based on a proxy involving several correlated features.

\subsubsection{FlipTest Equalized Odds}
\label{sec:ssl_eq}

\begin{figure}
\centering
\resizebox{\columnwidth}{!}{
\includegraphics{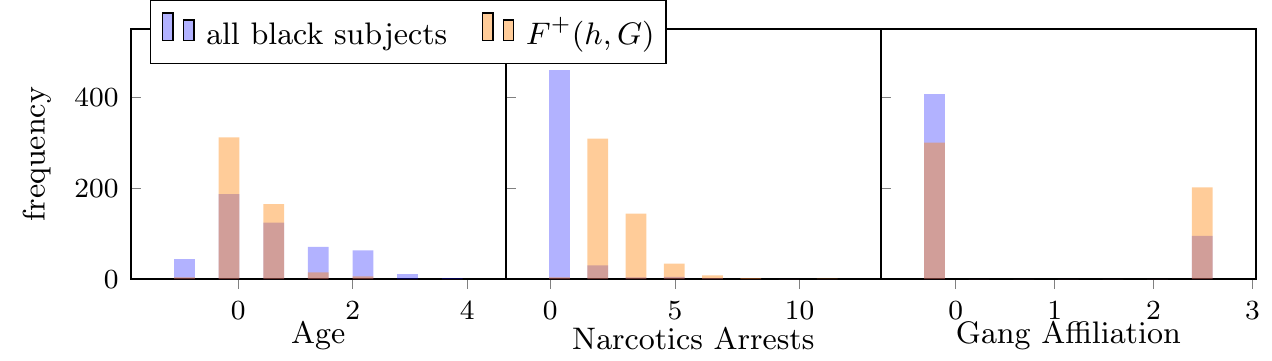}}
\caption{Distribution of the positive flipset for a black-to-white mapping of ground-truth negative subjects (orange) and the overall black ground-truth negative population (purple). The x-axis displays (normalized) feature values, and the y-axis displays frequency.}
\label{fig:SSL_eq_odd}
\end{figure}

We now analyze the SSL model, targeting a somewhat different fairness criteria related to equalized odds.
We do so by training two optimal transport approximations: one to map only black ground-truth negatives to white ground-truth negatives, and one to map only black ground-truth positives to white ground-truth positives.
Both GANS were trained with $\lambda=10^{-4}$.
Additional experimental details are given in
\ifarxiv
Appendix~\ref{app:experimental-setup}.
\else
Appendix~\ref{app:experimental-setup} of the supplementary material.
\fi

When mapping ground-truth negatives, we find that \posflip contains 499 individuals, out of 5,002 ground-truth negative black subjects marked as high risk by the model.
Meanwhile, \negflip was empty. 
Prompted by the size and asymmetry of the flipsets, we compare the marginals of \posflip compared to those of the overall black ground-truth negative population (Figure~\ref{fig:SSL_eq_odd}) to identify the subpopulation most affected by this potential discrimination.

\posflip largely consists of younger subjects who are much more likely to have narcotics arrests in their record, and be a part of a gang, but have a lower trend in recent criminal activity.
There is not much difference in the rest of the features.
The complete histograms of the marginals and the transparency report are given in Figure~\ref{fig:SSL_eq_app} in the
\ifarxiv
appendix.
\else
supplementary material.
\fi
Figure~\ref{fig:SSL_eq_app} also contains the transparency report, which has narcotics arrests at the top of both lists, i.e., narcotics arrests is both the largest average change and the most consistent change under the GAN mapping for individuals in \posflip. Thus, the transparency report correctly identifies the source of bias that was injected into the model.

When mapping ground-truth positives, we find that \posflip contains 216 individuals out of 1,568 black ground-truth positive subjects labeled as high risk, and that \negflip has 102 individuals out of 3,767 black ground-truth positive subjects labeled as low risk.
The transparency report for \posflip lists narcotics arrests as one of its main features, correctly identifying this source of bias.

The transparency report for \negflip lists a positive change in age as one of the most consistent features. This suggests that some older black subjects were mapped to younger white subjects, and since the model relies most heavily on age, the white counterparts were marked as high risk while the black individuals were marked as low risk.
Since the original SSL model, which generated the ground-truth labels, also relies heavily on age, the 629 ground-truth positive white subjects were all quite young, while the age distribution of the 5,335 ground-truth positive black subjects had a larger spread.
Thus, the transparency report on \negflip is a result of optimal transport, which inherently tries to match the generated distribution to the target distribution.
Practitioners should be cognizant of this distribution-matching behavior when mapping between starkly different distributions, as the model may be justified in using some features that are associated with the protected attribute.
This flipset points out the model's heavy reliance on age, and a domain expert may weigh in on whether or not this is justifiable.
Full results are available in Figure~\ref{fig:SSL_eq_app2} in the
\ifarxiv
appendix.
\else
supplementary material.
\fi

\subsection{Testing a Group-Fair Model}
\label{sec:lipton}
\begin{figure}
\centering
\resizebox{\columnwidth}{!}{\includegraphics{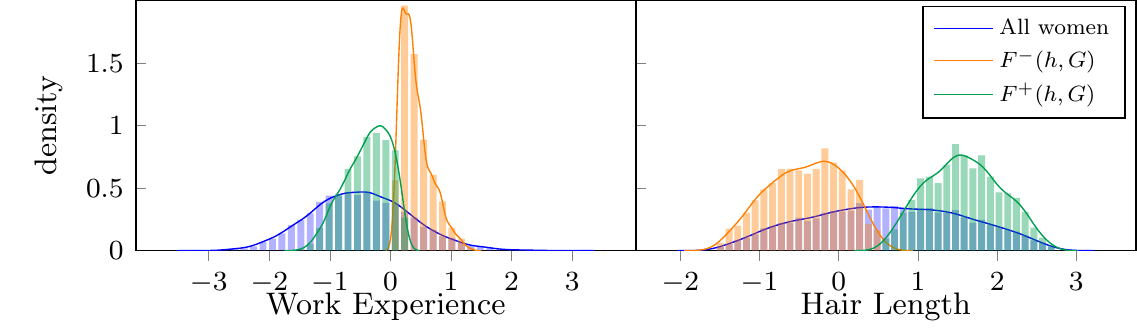}}
\caption{Distribution of women in the hiring data generated by Lipton et al.~\cite{lipton2018does} (Section~\ref{sec:lipton}). The green distribution is that of the women who were hired while their male statistical counterparts were not, and the orange distribution represents those who were not hired while their male statistical counterparts were, according to a model of Zafar et al.~\cite{zafar2017fairness-aistats} that seeks to equalize hiring rates across genders.}
\label{fig:lipton}
\end{figure}
Previously, Lipton et al.~\cite{lipton2018does} argued that some fair learning algorithms employ a ``problematic within-class discrimination mechanism''.
To support this argument, they create a synthetic data distribution that consists of two features: work experience and hair length.
They use this distribution to train a learning algorithm by Zafar et al.~\cite{zafar2017fairness-aistats} that seeks to equalize hiring rates across genders.
They then find that the resulting linear model uses hair length as a proxy for gender, thereby unfairly benefiting long-haired men and harming short-haired women.
In this section, we replicate this experimental setting to demonstrate that our method can detect unfair behavior in a model that appears to be fair at the population level.

We trained a linear model on 10,000 men and 10,000 women drawn from the male and female distributions, respectively, after scaling each feature to zero mean and unit variance.
Then, we trained approximate optimal transport mappings ($\lambda = 10^{-4}$) in both directions and evaluated the fairness of the model on a test set of 10,000 men and 10,000 women drawn from the same distributions.
Here we present the women-to-men mapping, and the results on the men-to-women mapping is given in
\ifarxiv
Appendix~\ref{app:backwards}.
\else
Appendix~\ref{app:backwards} of the supplementary material.
\fi

At the population level, the model treated the two groups similarly, hiring 27\% of the men and 30\% of the women.
However, when we mapped the women to the men 715 women were disadvantaged by the model (rejected with hired male counterpart) whereas 1215 were advantaged (hired with rejected male counterpart).
These flipsets comprise a much larger portion of the population than those encountered in Section~\ref{sec:validation}, suggesting some discriminatory behavior at the subgroup level.
Looking more closely at distributions of these flipsets (Figure~\ref{fig:lipton}) provides insight into the subgroups experiencing discrimination: the disadvantaged women tend to have shorter hair and longer work experience than the average woman, and the advantaged women have the opposite characteristics.
This is consistent with the observation of Lipton et al.~\cite{lipton2018does} that the model penalizes people with a masculine characteristic (short hair) in order to equalize hiring rates.

In addition, the transparency report shows that the disadvantaged women have much less (1.3 standard deviations) work experience and slightly longer (0.5 SD) hair than their counterparts, while the advantaged women have slightly less (0.6 SD) work experience and much longer (1.6 SD) hair than their counterparts.
This suggests that the disadvantaged women are disadvantaged due to their short work experience and that the advantaged women are advantaged due to their hair length.
We note that in general, this is not sufficient to establish \emph{causal} claims about why the model behaves this way; for example, the difference in hair length may be due to the model's use of some other feature that is correlated with hair length.
In this case, however, the model's weights---which define its causal behavior---and the result of the transparency report agree.
The model, which is a linear regressor, weights the two features almost equally ($w_{\it hair}=1.4$, $w_{\it work}=1.2$).
Since work experience, but not hair length, is a legitimate factor in most hiring decisions, after this deeper investigation into the model we may be able to conclude that this apparently fair model in fact discriminates against some men.


\subsection{Comparison with Counterfactual Fairness}
\label{sec:counterfactual}

\begin{figure}
\centering
\resizebox{\columnwidth}{!}{
\includegraphics{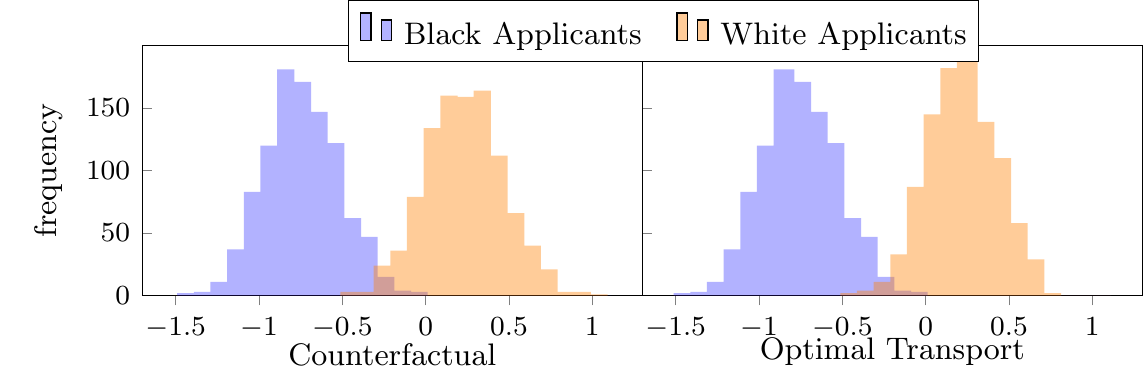}
}
\caption{Results of regression model predicting first year law school success on a) counterfactuals generated with a causal model and b) alternate inputs generated with optimal transport. The x-axis is predicted first year average at law school, and the y-axis is frequency.}
\label{fig:count1}
\end{figure}

\begin{figure}
\centering
\resizebox{0.8\columnwidth}{!}{
\includegraphics{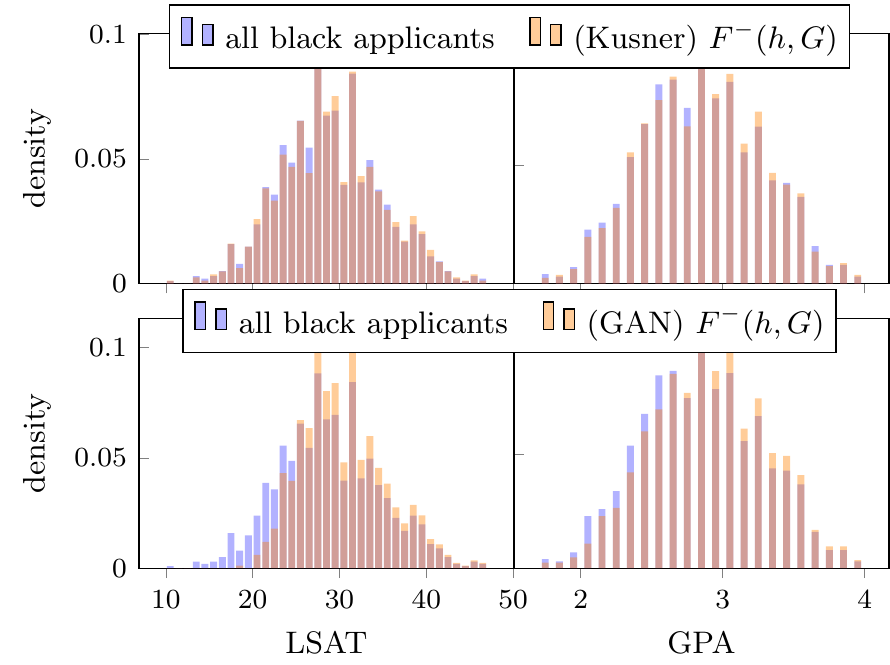}
}
\caption{Flipset distributions for the counterfactual (Kusner) and the GAN-generated data. Transparency report results are included in Figure~\ref{fig:counter_trans} in the
\ifarxiv
appendix.
\else
supplementary material.
\fi}
\label{fig:count2}
\end{figure}

In this section, we compare FlipTest with the auditing technique described by Kusner et al.~\cite{kusner2017counterfactual} on a law school dataset~\cite{wightman1998lsac}.
We test for fairness in two models: a regression model that predicts the first year average grade (FYA) in law school based on LSAT score, undergraduate GPA, race, and gender; and a binary classification model that predicts whether an individual will be in the top or bottom 50\% of class using the same attributes.
We create this second model by using the median FYA as a threshold.

Kusner et al.\ construct a causal model that postulates how race and gender are related to LSAT scores and GPAs, so that they can intervene on these attributes to construct counterfactuals.
Their model assumes that GPA is distributed according to $\mathcal{N}(b_G+w^{K}_GK+w^{R}_GR+w^{S}_GS, \sigma_G)$ and LSAT as Poisson($b_L+w^{K}_LK+w^{R}_LR+w^{S}_LS$), where $R$ is race, $S$ is sex, and $K$ is knowledge, an underlying parameter the model postulates is independent of race and sex.
$w^{R}_\star$ and $w^{S}_\star$ are weights for race and sex specific to the GPA and LSAT distributions, and knowledge is distributed as $\mathcal{N}(0,1)$.
Using this causal model, they create counterfactuals by intervening on the protected attributes, i.e., generating the underlying knowledge value for a section of the law school data, and then re-generating LSAT and GPA scores with flipped race attributes.

Kusner et al.\ use these counterfactual distributions to audit the law school regression model.
They determine that a model is counterfactually fair if the distributions of the model's prediction are identical for the pre-intervention and post-intervention inputs.
We perform this same test, using both their causally generated inputs and the pairs generated by an approximate optimal transport mapping ($\lambda=5 \times 10^{-5}$).
Since Kusner et al.\ do not normalize their data for their models or counterfactual generation process, we also do not normalize this data. 
We see in Figure~\ref{fig:count1} that the two methods of data generation give similar results on this test, both suggesting a conclusion of discrimination by the model.

We also provide a comparison with the FlipTest method, again using both methods of generating target distributions.
The flipset sizes are very similar for both models: both models have $|\posflip| = 0$, and $|\negflip|$ is 811 and 835 for the counterfactual and GAN data, respectively.
Figure~\ref{fig:count2} shows that the distributions of the flipsets are similar, but the GAN-generated inputs tend to lead to a flipset that contains black applicants with higher GPA and LSAT scores.
The transparency reports (Figure~\ref{fig:counter_trans} in the
\ifarxiv
appendix)
\else
supplementary material)
\fi
for the two methods also largely agree with each other, with LSAT ranked higher than GPA, and gender having essentially no effect. Since gender has little effect on the flipset, we do not plot the gender distributions of the two flipsets, and we do not report the transparency report results for the race feature, since by definition of the flipset, race will always change. 
However, other features can lead to bias as well---the transparency report clearly shows that LSAT score changes significantly under the GAN mapping ($>1$ standard deviation), and this may indicate another source of bias in the model.

We now investigate the weights of our model, comparing the transparency report to what we know about the model.
First, there is discrimination in the model based on race, as evidenced by the fact that the feature \emph{race=black} has a negative weight and the feature \emph{race=white} has a positive weight.
Second, after correcting for the fact that the data in this experiment is not normalized, we see that the weight for LSAT is higher than that for GPA.
The transparency report is consistent with this observation.

We conclude this section by observing that, even without access to a causal model, FlipTest can give nearly identical results as causally generated counterfactuals.


\subsection{Recognizing Biased Features}
\label{sec:fairtest}

Here we demonstrate a major conceptual difference of FlipTest from FairTest by Tram\`er et al.~\cite{tramer2017fairtest}.
FairTest searches for subgroups experiencing possible discrimination by using a decision tree to iteratively search for subgroups within which there is high correlation between the model output and the protected attribute.
It then reports the subgroups with has the highest correlations.

To clearly illustrate the difference between FairTest and FlipTest, we use a synthetic data distribution with only one feature: the number of prior arrests.
This feature is distributed as $\text{Geometric}(1/4) - 1$ for people in $\ds$ and $\text{Geometric}(1/2) - 1$ for people in $\ds'$.
In addition, we assume a simple model, which classifies a person as low risk if the number of prior arrests is zero, high risk if it is two or more, and either low risk or high risk uniformly at random if there is exactly one prior arrest.
Because the members of $\ds$ tend to have more prior arrests than those in $\ds'$, this model classifies disproportionately higher number of people in $\ds$ as high risk.

We let $S$ and $S'$ be sets of 10,000 points drawn from $\ds$ and $\ds'$, respectively, and ran FairTest and FlipTest on these sets.
FairTest correctly identified possible discrimination at the group level, reporting a confidence interval of $[0.2452, 0.2941]$ for the correlation between the model's output and the protected attribute.
However, the correlation decreased in all subgroups that FairTest subsequently considered.
This is because subgroups in FairTest are defined by the values of the input features.
Thus, FairTest compares a subgroup of $S$ to a subgroup of $S'$ \emph{that has similar numbers of arrests}.

On the other hand, FlipTest recognizes that the distributions of the number of arrests is different, and adjusts the comparisons accordingly.
For example, the set of people with 1 arrest in $S$ is compared to the set with 2--4 arrests in $S'$.
As a result, the flipsets are very large and unbalanced, with $|\posflip| = 2572$ and $|\negflip| = 0$.
In addition, the transparency report explains a reason for this difference, showing that 100\% of the people in $\posflip$ had more arrests than their statistical counterparts in $\ds'$, with an average of 1.44 more arrests.

We see this phenomenon on real data as well.
On the SSL data (Section~\ref{sec:ssl}), FairTest notes an overall bias against the entire black population, but does not report that the discrimination is based on narcotics arrest since the feature itself is biased, with higher levels for black subjects than white subjects.
The full results are given in Figure~\ref{fig:fairtest_ssl} in the
\ifarxiv
appendix.
\else
supplementary material.
\fi

Thus, the choice of the appropriate fairness test may depend on the setting.
If the number of prior arrests is a strong indicator for recidivism risk, then it makes sense to compare subgroups of people with similar numbers of arrests.
On the other hand, if the model had used a feature that is completely unrelated to crime, it would be harder to justify comparing people who are similar with respect to that feature.
Our experiments in this section show that FairTest is better suited for settings where the input features can justify potential differences in the model's output.

	\section{Related Work}
	\label{sec:related}

\textbf{Counterfactual Fairness Testing.}
Counterfactual fairness~\cite{kusner2017counterfactual} compares the model's behavior on a real input and a counterfactual, causally generated input.
Similarly, Datta et al.~\cite{datta2016algorithmic} perform causal interventions on input features to study which features are influential in changing in the output of the model, Wachter et al.~\cite{wachter2018counterfactual} generate simple $L_1$-nearest counterfactuals as a form of explanations for model outputs, and Ustun et al.~\cite{ustun2019actionable} develop a method that outlines what actions individuals can take to change their classification outcome in linear models.
However, these methods, like those discussed in the introduction~\cite{galhotra2017fairness, agarwal2018automated}, generate potentially unrealistic, out-of-distribution points, which can jeopardize their conclusions.
By contrast, the points that we generate conform to the data distribution.

\vspace*{1ex}
\noindent
\textbf{Optimal Transport.}
Others have proposed using the optimal transport map in the context of fairness, but to the best of our knowledge, it has not yet been used as a fairness testing mechanism.
Del Barrio et al.~\cite{delbarrio2018obtaining} use the optimal transport mapping to extend a previous method~\cite{feldman2015certifying} of data pre-processing, which obfuscates protected attribute information, to the multivariate setting.
Concurrent work from Yang et al.~\cite{yang2019scalable} also develop a method of approximating an (unbalanced) optimal transport mapping using GANs.
Their formulation is closely related to ours, but they do not consider its application to fairness testing.
Altschuler et al.~\cite{altschuler2017near} and Quanrud~\cite{quanrud2018approximating} present efficient methods for approximating the optimal transport mapping, but the resulting mappings are not defined for previously unseen data points.
By contrast, the mappings produced by Leygonie et al.~\cite{leygonie2019adversarial}, Perrot et al.~\cite{perrot2016mapping}, and Seguy et al.~\cite{seguy2017large} do generalize to unseen points, making them suitable for use with FlipTest.

\vspace*{1ex}
\noindent
\textbf{Individual Fairness.}
Individual fairness criteria~\cite{joseph2016fairness, zemel2013learning, dwork2018fairness,dwork2012fairness} bind guarantees about the fairness of a model's behavior to every individual, as opposed to an aggregated statistic.
Dwork et al.~\cite{dwork2012fairness} note that, in some cases, the model must essentially be a constant function to satisfy individual fairness and group fairness at the same time.
They propose an alternative that applies an optimal transport mapping to one of the groups, obtaining a transformed dataset on which they solve individual fairness.
FlipTest is motivated by this approach, but we specifically look for potentially discriminatory differences between pairs of individuals \emph{that belong to different groups}, and use optimal transport to construct pairs of individuals that exemplify these differences.

\vspace*{1ex}
\noindent
\textbf{Subgroup Fairness.}
One application of this work is uncovering \emph{subgroup unfairness}~\cite{kearns2019empirical,hebert2018multicalibration}, i.e., identifying subgroups that are possibly harmed as a result of their group membership.
FairTest~\cite{tramer2017fairtest} uses a decision tree to find subgroups with high discriminatory association, while taking care to ensure that the association is statistically significant.
However, as we show in Section~\ref{sec:fairtest}, FairTest does not handle the case where the input feature itself is biased.
Zhang et al.~\cite{zhang2016identifying} also find a computationally faster way to search the exponentially large number of subgroups, but their method also suffers from the same issue that FairTest does.
Kearns et al.~\cite{kearns2018preventing} prove that checking for subgroup fairness is equivalent to weak agnostic learning, which is computationally hard in the worst case.
FlipTest differs from these works in that we do not require the subgroups of interest to be specified before the fairness testing.

	\section{Conclusion}

FlipTest is a low-cost fairness testing framework that is sensitive to discrimination beyond group fairness metrics, is proficient at displaying unfair treatment in models with biased data, and can be used to as a first step towards detecting a model's method of discrimination.
As future work, extending the framework beyond binary classifiers, which are the most commonly studied case in the fairness literature, and exploring the application of FlipTest to uncovering additional types of discrimination are both promising directions.


\ifarxiv
\subsection*{Acknowledgment}
This material is based upon work supported by the National Science Foundation under Grant No.~CNS-1704845.
\fi

	\bibliographystyle{plainnat}
	\bibliography{biblio}
	
	\newpage
	\appendix

\section{Proofs}
\label{app:proof}
\setcounter{proposition}{0}
\begin{proposition}
	Suppose that $G^*$ is a minimizer of $L_G$ among all $G$ such that $G(S) = S'$.
	If $\lambda > 0$, $G^*$ is an exact optimal transport mapping from $S$ to $S'$.
\end{proposition}
\begin{proof}
	If $G(S) = S'$, $G$ induces a bijection between $S$ and $S'$, so we have $\frac{1}{n} \sum_{\boldsymbol{x} \in S} D(G(\boldsymbol{x})) = \frac{1}{n} \sum_{\boldsymbol{x}' \in S'} D(\boldsymbol{x}')$.
	Then, \eqref{eqn:loss} becomes
	\[ \textstyle L_G = \frac{1}{n} \sum_{\boldsymbol{x}' \in S'} D(\boldsymbol{x}') + \frac{\lambda}{n} \sum_{\boldsymbol{x} \in S} c(\boldsymbol{x}, G(\boldsymbol{x})). \]
	The first term does not depend on $G$, so if $G^*$ minimizes $L_G$, it also minimizes $\frac{\lambda}{n} \sum_{\boldsymbol{x} \in S} c(\boldsymbol{x}, G(\boldsymbol{x}))$.
	Since $\lambda > 0$, $G$ then minimizes $\frac{1}{n} \sum_{\boldsymbol{x} \in S} c(\boldsymbol{x}, G(\boldsymbol{x}))$, which is exactly the definition of an optimal transport mapping.
\end{proof}

\begin{proposition}
	Let $h$ be a binary classifier and $G: S \to S'$, with $|S| = |S'|$, be the exact optimal transport mapping.
	Then, $|\posflip| = |\negflip|$ if and only if the model satisfies demographic parity on the observed points, i.e.,
	\[ |\{\boldsymbol{x} \in S \mid h(\boldsymbol{x}) = 1 \}| = |\{\boldsymbol{x}' \in S' \mid h(\boldsymbol{x}') = 1 \}|. \]
\end{proposition}
\begin{proof}
	We have
	\begin{align*}
		&|\posflip| \\
		&= |\{\boldsymbol{x} \in S \mid h(\boldsymbol{x}) = 1 \wedge h(G(\boldsymbol{x})) = 0\}| \\
		&= |\{\boldsymbol{x} \in S \mid h(\boldsymbol{x}) = 1\}| - |\{\boldsymbol{x} \in S \mid h(\boldsymbol{x}) = 1 \wedge h(G(\boldsymbol{x})) = 1\}|,
	\end{align*}
	and similarly we have
	\begin{align*}
		&|\negflip| \\
		&= |\{\boldsymbol{x} \in S \mid h(\boldsymbol{x}) = 0 \wedge h(G(\boldsymbol{x})) = 1\}| \\
		&= |\{\boldsymbol{x} \in S \mid h(G(\boldsymbol{x})) = 1\}| - |\{\boldsymbol{x} \in S \mid h(\boldsymbol{x}) = 1 \wedge h(G(\boldsymbol{x})) = 1\}|.
	\end{align*}
	Therefore, the equality $|\posflip| = |\negflip|$ is equivalent to $|\{\boldsymbol{x} \in S \mid h(\boldsymbol{x}) = 1\}| = |\{\boldsymbol{x} \in S \mid h(G(\boldsymbol{x})) = 1\}|$, and it remains to show that the right-hand side of the last equation equals $|\{\boldsymbol{x}' \in S' \mid h(\boldsymbol{x}') = 1 \}|$.
	This follows from letting $\boldsymbol{x}' = G(\boldsymbol{x})$ and the fact that $G$ is a bijection between $S$ and $S'$.
\end{proof}

\section{Experimental Control}
\label{app:control}

Here we give more details about the control experiment described in Section~\ref{sec:validation}.
We trained a fair model by ensuring that, from the perspective of the model, there are no distributional differences between the two protected groups.
More specifically, we set $\ds$ and $\ds'$ respectively to the normal distributions $\dn(\bm{\mu}_{\ds}, \mathbf{I}_6)$ and $\dn(\bm{\mu}_{\ds'}, \mathbf{I}_6)$, where $\bm{\mu}{}_{\ds} = (0, 0, 0, 1, 1, 1)^T$, $\bm{\mu}_{\ds'} = (0, 0, 0, -1, -1, -1)^T$, and $\mathbf{I}_6$ is the 6x6 identity matrix.
Thus, the first three features had the same joint distribution for both groups, and the model was allowed to use only these three features.

We constructed the training set by drawing 10,000 points from each of $\ds$ and $\ds'$.
To make the model arbitrary and complicated, we gave each point a binary class label uniformly at random and then trained with this data an SVM classifier with a radial basis function kernel.
Finally, we trained a GAN ($\lambda = 10^{-4}$) using a test set, which was drawn from $\ds$ and $\ds'$ in the same way as the training set, and observed the size of the flipset on the test set.

Our GAN correctly mapped the distribution $\ds$ to $\ds'$, changing each of the first three features by less than 0.01 on average and the other three by close to 2 on average.
As a result, even though the model had a very irregular decision boundary, with approximately 63\% training accuracy (the test accuracy was, of course, close to 50\%) and 53\% positive classification rate, the flipsets were small.
Only 167 of the \textasciitilde5,300 points classified as positive by the model flipped from positive to negative under the GAN mapping, and 148 of the \textasciitilde4,700 points classified as negative by the model flipped from negative to positive.
These numbers stayed similar when we altered the distributions $\ds$ and $\ds'$ so that the features were not necessarily normal and uncorrelated.

\section{GAN Experimental Setup Details}
\label{app:experimental-setup}
Our GAN experimental setup is the same in all experients in Section~\ref{sec:application} except $\lambda$, the $L_1$ constraint on generator output, and the batch size. We describe the general setup first. We use a Wasserstein GAN~\cite{arjovsky2017wasserstein} using Keras~\cite{keras} with a Tensowflow backend~\cite{tensorflow} on Python~3. Our Wasserstein GAN has a generator and a critic both with two dense layers of size 128; the critic has an output layer of size 1. We use ReLU activations on both discriminator and generator. We use the RMSProp optimizer with a learning rate of $5 \times 10^{-5}$. We train for 20,000 epochs, have a weight clip value of 0.01, and train the critic 5 times more than we train the generator.

For all experiments with the SSL dataset, we train the GAN on 1/4 of the SSL data~\cite{chicago}, separated into white and black groups, with only the 8 features mentioned in Section~\ref{sec:ssl}. This corresponds to 41,560 black subjects and 16,465 white subjects. For the demographic parity-based experiments, $\lambda = 5 \times 10^{-4}$ and the batch size is 4. For the equalized odds based experiments, $\lambda = 1 \times 10^{-4}$ and the batch size is 4. For all SSL experiments, we use random seed 100.

For the experiments on the Lipton et al.~\cite{lipton2018does} dataset, we have $\lambda = 10^{-4}$ and a batch size of 64. As described in Section~\ref{sec:lipton}, this GAN is trained on 10,000 men and 10,000 women generated from the synthetic data in~\cite{lipton2018does}. All experiments in Section~\ref{sec:lipton} use random seed 0.

For the experiments in Section~\ref{sec:counterfactual}, we use the law school success data~\cite{wightman1998lsac}, as preprocessed by Kusner et al.~\citep{kusner2017counterfactual}. We train a GAN with $\lambda = 5 \times 10^{-5}$ and batch size 2. The GAN is trained on 14,649 white applicants and 1008 black applicants. All experiments in Section~\ref{sec:counterfactual} use random seed 100.

We search for optimal hyper-parameters ($\lambda$ and batch size) by looking at 
\begin{itemize}
\item The results of Kolomorov-Smirnoff two sample tests between feature-wise marginal distributions of the target and the generated distributions. We aim to make the reported statistic as small as possible.
\item The difference in mean squared error between real and generated samples on a set of regression models, trained on the target distribution, one for each feature in the dataset, that predict that feature from the rest of the features in the dataset. 
\item The difference in the average squared $L_1$ distance between mapped pairs in the GAN generated mapping versus a 2,000-element sample of the exact optimal transport mapping.
\end{itemize}

\section{Training GANs in Opposite Direction}
\label{app:backwards}
We continue the discussion started in Section~\ref{sec:flipset}, with the same running example of a distribution of women $\ds$ and a distribution of men $\ds'$, who are inputs to a hiring classifier.
Note that under an exact optimal transport mapping, the map $f: \ds \to \ds'$ from female to male applicants and $f' : \ds' \to \ds$ from male to female applicants would be inverses---so the \posflipf of hired women mapped to unhired men and \negflippf of unhired men mapped to hired women would correspond to the same statistical pairs of individuals, and thus the sets would have the same size. (This would also hold for the \negflipf and \posflippf).

Given this behavior for flipsets based on exact optimal transport mappings, for two GAN mappings $G: \ds \to \ds'$ and $G': \ds' \to \ds$, we would expect that if the GANs are approximating the optimal transport mapping well, the size of \posflip (hired women mapped to unhired men) should be similar to the size of \negflipp (unhired men mapped to hired women).
Similarly, we would expect $|\negflip| \approx |\posflipp|$.
Thus, if we find similar sized flipsets across both directions of translations in our GAN approximations, then this gives us more confidence that the GAN is producing trustworthy output that approximates the optimal transport mapping well.

In Table~\ref{tbl:reverse_maps}, we provide the results of reverse mappings for all FlipTest experiments Section~\ref{sec:application}, along with the hyper-parameters used to train them.
We note that aside from the $\lambda$ value and batch size, the rest of the experimental setup was identical to that outlined in Appendix~\ref{app:experimental-setup}.
Overall, we find that the reverse mappings are consistent with the forward mapping presented in Section~\ref{sec:application}.

\begin{table}
\caption{Results of flipset sizes based on GAN mappings that are in the other direction than what is presented in the paper. $|\posflip|$ and $|\negflip|$ refer to the sizes of the flipsets that are presented in the main paper, $|\negflipp|$ and $|\posflipp|$ refer to the sizes of the flipsets based on a mapping in the reverse direction, and $\lambda$ and $b$ refer to the hyper-parameters $\lambda$ and batch size, respectively, used to create the GAN mapping in the reverse direction. We note that for the Equalized Odds (Pos) experiment, there were only 659 instances to train on, so we don't expect the model to learn a good mapping, which was reflected in the numbers generated in both directions. Additionally, for the Law School dataset, the simple model we used predicted no black students passing, so the large flipset presented is in fact in line with the model's predictions.}
\label{tbl:reverse_maps}
\centering
\resizebox{\columnwidth}{!}{%
\begin{tabular}{lrrrrrr}
 Experiment & \rotatebox{90}{$|\posflip|$} & \rotatebox{90}{$|\negflip|$} & \rotatebox{90}{$|\negflipp|$} & \rotatebox{90}{$|\posflipp|$} & \multicolumn{1}{c}{$\lambda$} & $b$\\\hline
{SSL: Dem Parity}
           & $1290$ & $4$ &  $815$ & $3$ & $2\times10^{-4}$ & $16$ \\
{SSL: Eq Odds (Neg)}
           & $499$ & $0$ &  $742$ & $1$ & $2\times10^{-4}$ & $4$ \\
{SSL: Eq Odds (Pos)}
           & $216$ & $102$ & $0$ & $190$ & $1\times10^{-4}$ & $4$\\
{Lipton: Dem Parity}
           & $715$ & $1215$ & $774$ & $1075$ & $1\times10^{-4}$ & $64$ \\
{Law School}
           & $0$ & $835$ & $0$ & $8943$ & $5\times10^{-5}$ & $2$
\end{tabular}
}
\end{table}



\begin{table}[h]
\caption{Full GAN validation results for all GANS used in Section~\ref{sec:application}. KS refers to the Kolmogorov--Smirnov two-sample test statistic. MSE Diff refers to the difference, between real and generated data, of the mean squared error of a linear regression model trained on the real data to predict a feature from the rest. For KS and MSE, we ran 10 trials; the mean is given first, with standard deviation in parentheses. Dist: OT, GAN refers to the average of the squared $L_1$ distance between a data point $x$ and its counterpart $G(x)$ under an exact optimal transport mapping (OT) or an approximated GAN mapping (GAN).}
\label{tbl:validations_app}
\centering
\resizebox{\columnwidth}{!}{%
\begin{tabular}{llrrr}
Experiment &  Features & KS (std) & MSE Diff (std) & Dist: OT, GAN\\\hline
\multirow{8}{*}{\shortstack[l]{SSL: \\ Dem Parity}}
           & Age         & $0.054$ $(0.001)$ &  $0.070$ $(0.048)$ & \multirow{8}{*}{$2.04$, $1.64$} \\
           & Vic Shoot   & $0.003$ $(0.003)$ &  $0.108$ $(0.125)$ & \\
           & Vic Assault & $0.010$ $(0.002)$ &  $0.259$ $(0.054)$ & \\
           & Vio Arr     & $0.005$ $(0.003)$ &  $0.065$ $(0.045)$ & \\
           & Gang Aff    & $0.018$ $(0.006)$ &  $0.094$ $(0.034)$ & \\
           & Narc Arr    & $0.025$ $(0.002)$ &  $0.298$ $(0.058)$ & \\
           & Trend       & $0.142$ $(0.009)$ &  $0.368$ $(0.073)$ & \\
           & UUW Arr     & $0.007$ $(0.002)$ & $-0.103$ $(0.040)$ & \\\hline
\multirow{8}{*}{\shortstack[l]{SSL: \\ Eq Odds (Neg)}}
           & Age         & $0.019$ $(0.021)$ &  $0.019$ $(0.057)$ & \multirow{8}{*}{$1.52$, $1.42$} \\
           & Vic Shoot   & $0.001$ $(0.000)$ & $-0.024$ $(0.000)$ & \\
           & Vic Assault & $0.004$ $(0.001)$ & $-0.073$ $(0.053)$ & \\
           & Vio Arr     & $0.003$ $(0.002)$ &  $0.015$ $(0.068)$ & \\
           & Gang Aff    & $0.007$ $(0.004)$ & $-0.009$ $(0.053)$ & \\
           & Narc Arr    & $0.019$ $(0.007)$ &  $0.094$ $(0.078)$ & \\
           & Trend       & $0.151$ $(0.014)$ &  $0.083$ $(0.045)$ & \\
           & UUW Arr     & $0.007$ $(0.001)$ & $-0.049$ $(0.075)$ & \\\hline
\multirow{8}{*}{\shortstack[l]{SSL: \\ Eq Odds (Pos)}}
           & Age         & $0.007$ $(0.010)$ & $-0.019$ $(0.002)$ & \multirow{8}{*}{$5.40$, $5.09$} \\
           & Vic Shoot   & $0.007$ $(0.003)$ & $-0.214$ $(0.076)$ \\
           & Vic Assault & $0.023$ $(0.000)$ &  $0.072$ $(0.137)$ \\
           & Vio Arr     & $0.019$ $(0.009)$ &  $0.512$ $(0.183)$ \\
           & Gang Aff    & $0.009$ $(0.002)$ & $-0.048$ $(0.025)$ \\
           & Narc Arr    & $0.022$ $(0.006)$ &  $0.247$ $(0.017)$ \\
           & Trend       & $0.202$ $(0.021)$ &  $0.438$ $(0.109)$ \\
           & UUW Arr     & $0.036$ $(0.069)$ & $-1.000$ $(0.473)$ \\\hline
\multirow{2}{*}{\shortstack[l]{Lipton: \\ Dem Parity}}
           & Work Exp.   & $0.072$ $(0.014)$ &  $0.150$ $(0.204)$ & \multirow{2}{*}{$2.19$, $2.09$} \\
           & Hair Len.   & $0.074$ $(0.043)$ &  $0.141$ $(0.114)$ \\\hline
\multirow{4}{*}{\shortstack[l]{Law School}}
           & LSAT        & $0.057$ $(0.012)$ &  $3.757$ $(0.462)$ & \multirow{4}{*}{$10.99$, $11.10$} \\
           & GPA         & $0.110$ $(0.027)$ & $-0.040$ $(0.005)$ \\
           & Gender=M    & $0.172$ $(0.061)$ &  $0.000$ $(0.000)$ \\
           & Gender=F    & $0.168$ $(0.057)$ &  $0.000$ $(0.000)$
\end{tabular}
}
\end{table}


\section{Comparison with FairTest}
\label{app:fairtest}
\begin{table}[h]
\caption{Results of FairTest~\citep{tramer2017fairtest} on the biased classifier from Section~\ref{sec:ssl}. Notably, FairTest does not report much bias based on narcotics arrest, while FairTest does. This is due to differences between FlipTest and FairTest, in that FairTest compares black and white individuals with feature values in the same range while FlipTest does not. This phenomenon is expanded on in Section~\ref{sec:fairtest}.
}
\label{fig:fairtest_ssl}
\centering
\begin{tabular}{lr}
\textbf{Subgroup}
& \textbf{Correlation} \\ \hline

All Black Population
&  $[-0.0688, -0.0523]$ \\ \hline

Age $\in (-0.94, -0.16]$, Trend $\in (-\infty, 0.50)$
& $[-0.1646, -0.1164]$ \\ \hline

Age $\in (-\infty, -0.94]$, Trend $\in (-0.98, -0.50)$, \\ Gang Aff. $\in (-\infty, 0.91)$
& $[0.0969, 0.1883]$ \\ \hline

Vio Off. $\in (-\infty, 0.95]$, Age $\in (-\infty, -0.94]$, \\
Trend $\in (-0.24, \infty)$, Gang Aff. $\in (-\infty, 0.91)$ 
& $[0.0953, 0.1534]$ \\ \hline

Vio Off. $\in (-\infty, 0.95]$, Age $\in (-\infty, -0.94]$, \\
Trend $\in (-0.50, \infty)$, Narc Arr.  $\in (2.6, \infty)$
& $[0.0943, 0.1512]$ \\ \hline

Vio Off. $\in (-\infty, 0.95]$, Age $\in (-0.16, \infty]$, \\
Trend $\in (-0.50, \infty)$, Gang Aff. $\in (-\infty, 0.91)$
& $[0.0853, 0.3012]$ \\ \hline

Age $\in (-\infty, -0.16]$, Trend $\in (-\infty, -0.50)$
& $[-0.1203, -0.0804]$ \\ \hline 

Trend $\in (\infty, -0.50)$
& $[-0.1017, -0.0727]$ \\ \hline

Vio Off. $\in (-\infty, 0.95]$, Age $\in (-0.95, -0.16]$, \\
Trend $\in (-0.50, \infty)$,
& $[-0.0972, -0.0610]$
\end{tabular}
\end{table}
FairTest largely reports bias against middling-aged and younger black individuals with lower to middling trend in criminal activity, low gang affiliation and low history of violent crime.

FlipTest agrees with FairTest's focus on younger black subjects, but other major feature that FlipTest isolates as defining the disadvantaged subgroup is narcotics arrests. The discrepancy in these results comes from the fact that the discrimination in this model is based on the use of a feature where higher values are correlated with race, and so a FlipTest type comparison between individuals with \emph{different} levels of narcotics arrests is more appropriate to uncover this. This is not to say that the results from FairTest are invalid; it simply uncovers bias of a different type.

\section{SSL case study extension: multiple biased features}
\label{app:ssl}
We also provide a test of a biased classifier that relies on a combination of features to be discriminatory. We choose the three features that are the most positively correlated with race and least correlated with a history of violent crime: victim of shooting incident, victim of assault, and narcotic arrests. We set $w_{\it age} = 50$, $w_{\it vs} = 20$, $w_{\it va} = 20$, and $w_{\it narc} = 20$. We give all other features weight 0. We use the same GAN and classifier calibration procedure as in the first biased classifier experiment. We also scale the features in the same way, to zero mean and unit variance.
We find that there are 1,968 high-risk black subjects that are mapped to low-risk white counterparts (out of 3,683 total high-risk black subjects), and 0 low-risk black subjects that are mapped to high-risk white counterparts. Again, the extreme asymmetry of \posflip and \negflip along with the size of \posflip suggest there is some discrimination in the model.

The marginal distributions can be seen in Figure~\ref{fig:SSL2}: young, black individuals with nonzero accounts of being a victim of a shooting and assault, with nonzero narcotic arrests, are discriminated against by this model. They are also more likely to be in a gang and to not have a record of unlawful use of weapon arrests. 

Upon viewing the transparency report to determine which features may be a worth further investigation into the model's workings, we see that being a victim of assault, of a shooting, and having a history of narcotic arrests are the features with the highest mean difference in the flipset. Trend in recent criminal activity is the most consistent, negative change, but it is closely followed by narcotic arrests and then victim of assault and shooting. Thus, again, the model's use of the features align with the correlations shown in the transparency report.

\clearpage

\begin{figure*}
	\begin{subfigure}{\textwidth}
		\centering
		\resizebox{\textwidth}{!}{
			\includegraphics{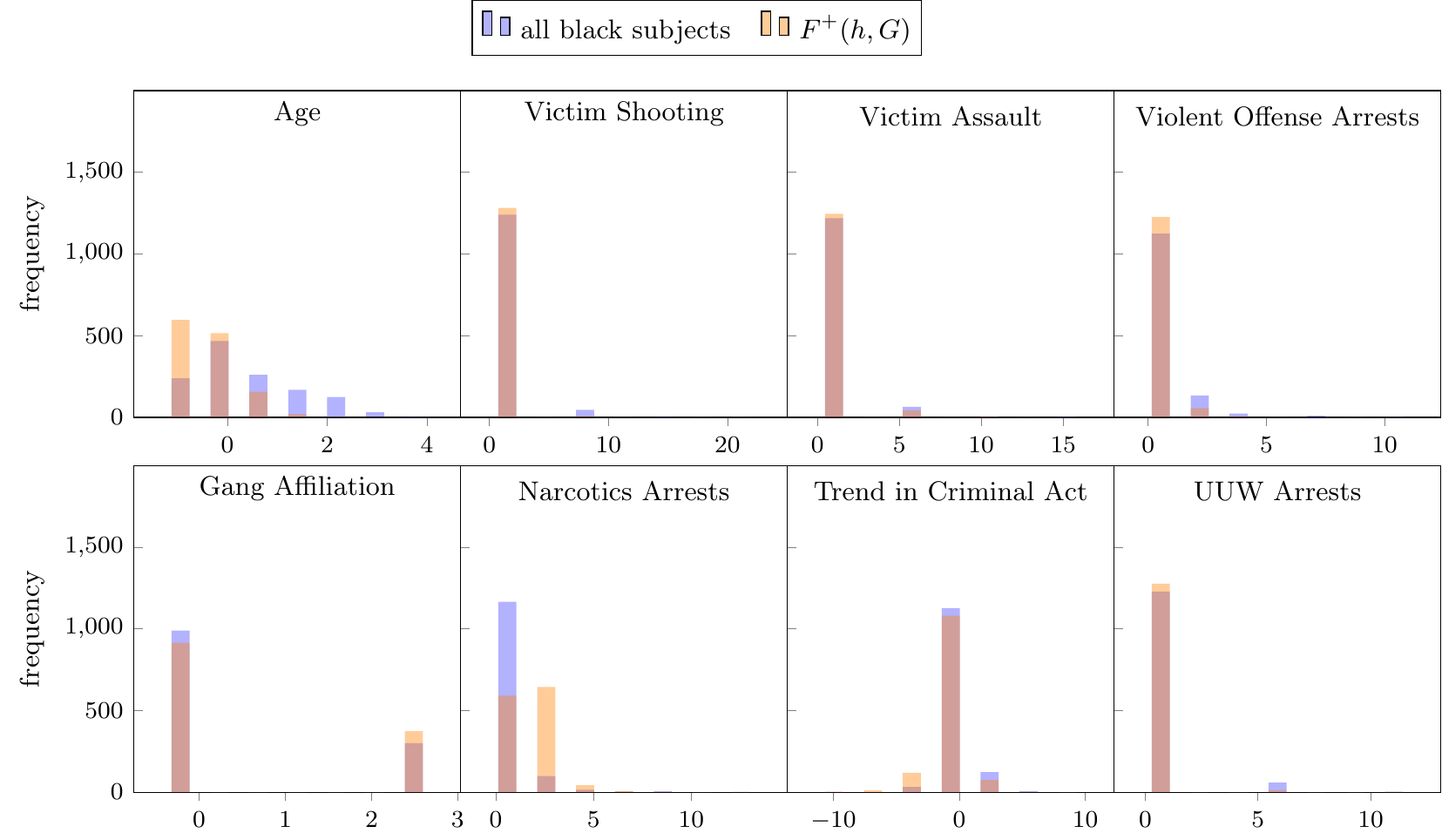}
		}
	\end{subfigure}
	\begin{subfigure}{\textwidth}
		\centering
		\begin{tabular}{lr}
			Feature      & Mean Sign \\ \hline
			Narcotic Arr & $0.99$    \\
			Trend        & $-0.78$   \\
			Gang Aff     & $0.26$    \\
			Vio Off      & $0.04$    \\
			Vic Assault  & $0.03$    \\
			Age          & $-0.01$   \\
			UUW Arr      & $0.00$    \\
			Vic Shooting & $0.00$
		\end{tabular}
		\quad
		\begin{tabular}{lr}
			Feature      & Mean Diff \\ \hline
			Narcotic Arr & $1.32$    \\
			Gang Aff     & $0.71$    \\
			Trend        & $-0.20$   \\
			Vic Assault  & $0.15$    \\
			Vio Off      & $0.11$    \\
			UUW Arr      & $0.20$    \\
			Age          & $-0.01$   \\
			Vic Shooting & $0.01$
		\end{tabular}
	\end{subfigure}
	\caption{All marginals and transparency report of the biased SSL classifier that puts weight only on age and narcotic arrests, from Section~\ref{sec:ssl_dem}. The plots show the distribution of the entire black subject population in comparison to that of \posflip, those that are labeled as high risk when their counterparts are not. The differences between these two distributions shed light on what subgroup may be treated differently by the model. We see that age and narcotic arrests have the largest departure from the rest of the population, while features like gang affiliation and being the victim of assault generally match the distribution of the overall population.
		Thus, the transparency report identifies the feature that the model relies on to bring about bias: narcotic arrests. Recall that all features are scaled to zero mean and unit variance.}
	\label{fig:SSL1.2}
\end{figure*}

\begin{figure*}
	\begin{subfigure}{\textwidth}
		\centering
		\resizebox{\textwidth}{!}{
			\includegraphics{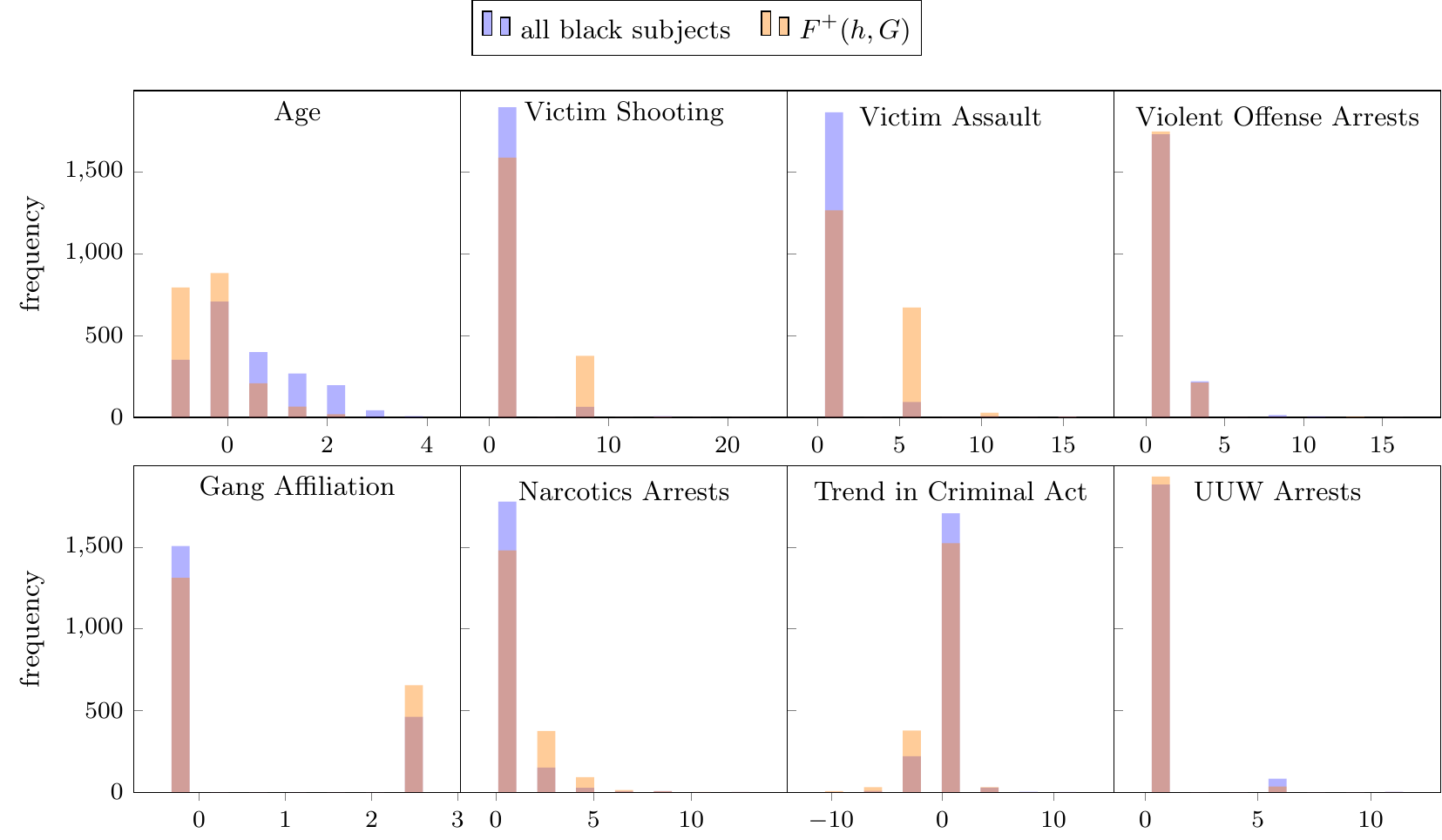}
		}
	\end{subfigure}
	\begin{subfigure}{\textwidth}
		\centering
		\begin{tabular}{lr}
			Feature      & Mean Sign \\ \hline
			Trend        & $-0.54$   \\
			Narcotic Arr & $0.43$    \\
			Vic Assault  & $0.39$    \\
			Vic Shooting & $0.23$    \\
			Gang Aff     & $0.10$    \\
			Vio Offense  & $0.09$    \\
			UUW Arr      & $0.01$    \\
			Age          & $-0.01$
		\end{tabular}
		\quad
		\begin{tabular}{lr}
			Feature      & Mean Diff \\ \hline
			Vic Shooting & $1.86$    \\
			Vic Assault  & $1.50$    \\
			Narcotic Arr & $0.62$    \\
			Gang Aff     & $0.26$    \\
			Vio Offense  & $0.22$    \\
			Trend        & $-0.09$   \\
			UUW Arr      & $0.05$    \\
			Age          & $-0.01$
		\end{tabular}
	\end{subfigure}
	\caption{Distribution of the high risk-to-low risk flipset for a black to white mapping for a model the model with multiple biased features in Appendix~\ref{app:ssl}. We did not plot the flipset in the other direction, as it was empty. We see that this model adversely affects young black individuals who have been the victim of shooting and assault incidents, who largely have nonzero narcotic arrests and have a slightly higher chance of being in a gang. Note that the distribution over violent offenses, and several other features, is relatively unchanged between the regular distribution and the flipset, suggesting that these features do not define the flipset distribution.
		The transparency report gives a strong case to investigate narcotic arrests, being a victim of a shooting, or assault, and possibly trend in recent criminal activity as features to investigate for causal influence. Three of these features were the ones we chose to overweight to make the model unfair.}
	\label{fig:SSL2}
\end{figure*}

\begin{figure*}
\begin{subfigure}{\textwidth}
\centering
\resizebox{\textwidth}{!}{
\includegraphics{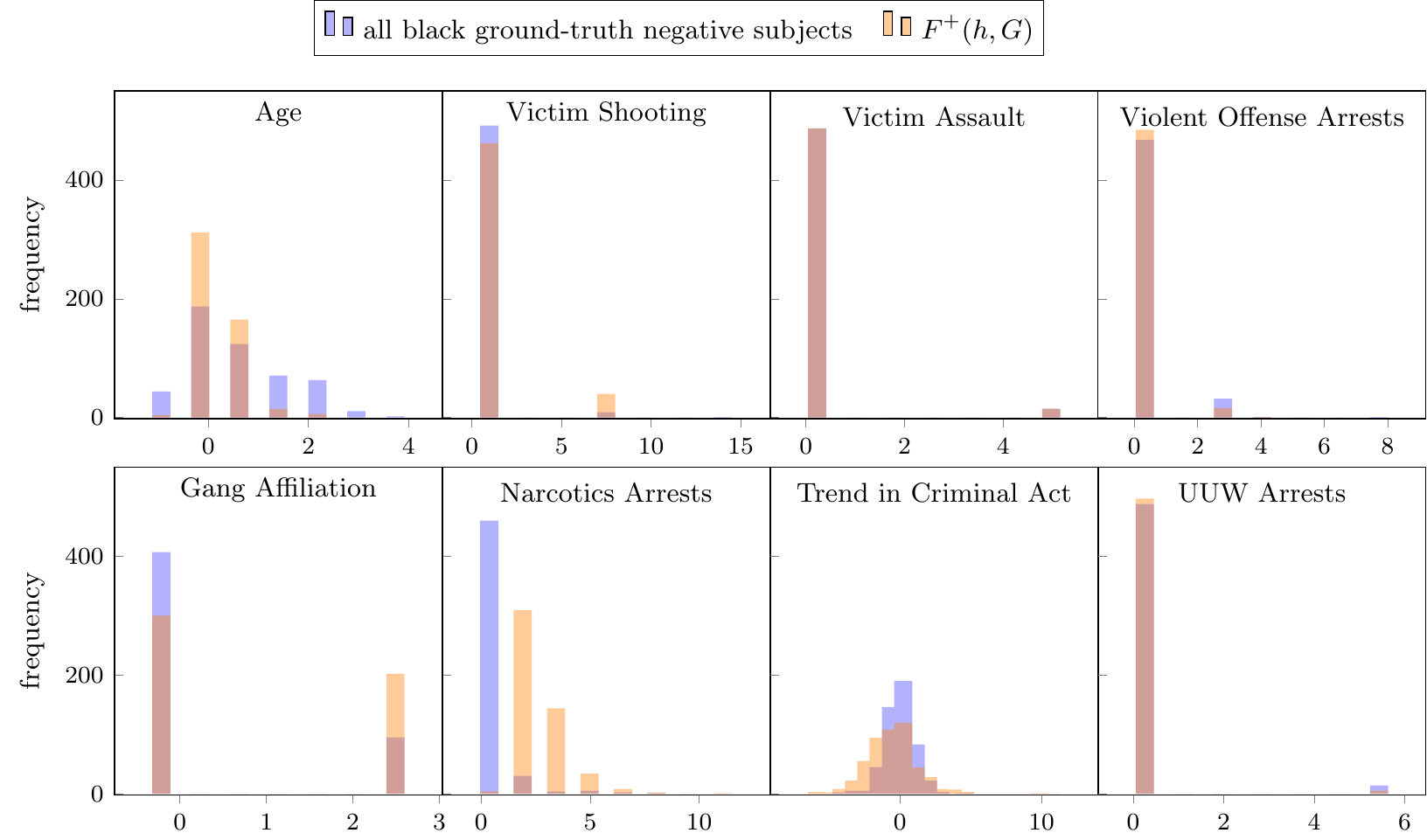}
}
\end{subfigure}
\begin{subfigure}{\textwidth}
\centering
\begin{tabular}{lr}
Feature      & Mean Sign \\ \hline
Narcotic Arr & $1.00$    \\
Trend        & $-0.85$   \\
Gang Aff     & $0.31$    \\
Vic Shooting & $0.07$    \\
Vic Assault  & $0.03$    \\
Vio Off      & $0.03$    \\
Age          & $0.00$    \\
UUW Arr      & $0.00$

\end{tabular}
\quad
\begin{tabular}{lr}
Feature      & Mean Diff \\ \hline
Narcotic Arr & $1.31$    \\
Gang Aff     & $0.84$    \\
Vic Shooting & $0.52$    \\
Trend        & $-0.36$   \\
Vic Assault  & $0.14$    \\
Vio Off      & $0.06$    \\
UUW Arr      & $0.05$    \\
Age          & $0.00$
\end{tabular}
\end{subfigure}
\caption{Flipset distribution and transparency report on \posflip for equalized odds testing on ground-truth negatives, from Section~\ref{sec:ssl_eq}.}
\label{fig:SSL_eq_app}
\end{figure*}

\begin{figure*}
\begin{subfigure}{\textwidth}
\centering
\resizebox{\textwidth}{!}{
\includegraphics{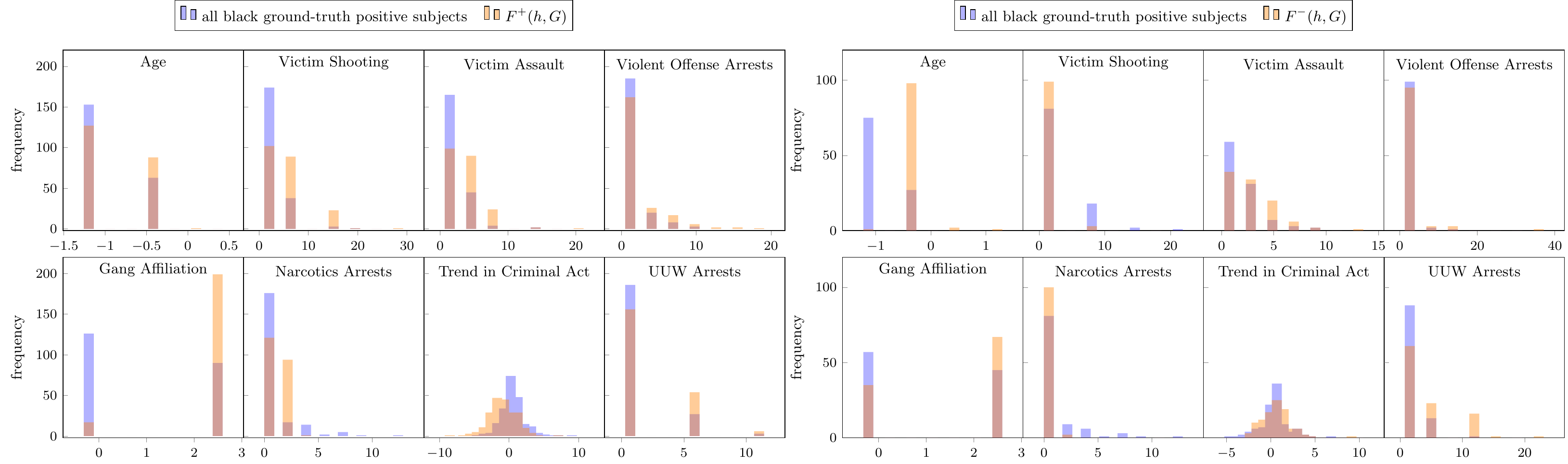}
}
\subcaption{Flipset distributions on \posflip and \negflip, respectively,}
\end{subfigure}

\begin{subfigure}{\textwidth}
\centering
\begin{tabular}{lr}
Feature      & Mean Sign \\ \hline
Narcotic Arr & $0.99$  \\
Trend        & $-0.87$ \\
Gang Aff     & $0.50$  \\
Vic Shooting & $0.44$  \\
Vio Off      & $0.27$  \\
UUW Arr      & $0.15$  \\
Vic Assault  & $0.13$  \\
Age.         & $-0.03$ \\
\end{tabular}
\quad
\begin{tabular}{lr}
Feature      & Mean Diff \\ \hline
Vic Shooting & $3.45$  \\
Gang Aff     & $1.37$  \\
Narcotic Arr & $1.19$  \\
Vio Off      & $0.89$  \\
UUW Arr      & $0.83$  \\
Vic Assault  & $0.73$  \\
Trend        & $-0.70$ \\
Age          & $-0.02$ \\
\end{tabular}
\subcaption{Transparency report on \posflip from the equalized odds test in Section~\ref{sec:ssl_eq}, mapping ground-truth positives to ground-truth positives.}
\quad
\begin{tabular}{lr}
Feature      & Mean Sign \\ \hline
Age          & $1.00$  \\
Gang Aff     & $0.61$  \\
UUW Arr      & $0.40$  \\
Vic Shooting & $0.32$  \\
Vio Off      & $0.20$  \\
Trend        & $-0.11$ \\
Vic Assault  & $0.02$  \\
Narcotic Arr & $0.00$  \\
\end{tabular}
\quad
\begin{tabular}{lr}
Feature      & Mean Diff \\ \hline
UUW Arr      & $3.10$  \\
Vic Assault  & $2.04$  \\
Gang Aff     & $1.68$  \\
Age          & $0.79$  \\
Vio Off      & $0.56$  \\
Vic Shooting & $0.19$  \\
Trend        & $-0.10$ \\
Narcotic Arr & $0.00$  \\
\end{tabular}
\subcaption{Transparency report on \negflip from the equalized odds test in Section~\ref{sec:ssl_eq}, mapping ground-truth positives to ground-truth positives.}
\end{subfigure}
\caption{Here we display the full results from Section~\ref{sec:ssl_eq}, where we test a GAN with FlipTest based on equalized odds. The flipset distribution and transparency report largely agree with the discrimination picked up in the demographic parity test, identifying narcotics arrests as a source of discrimination. We present results for \negflip as well, but we note that these results are of borderline significance since the negative flipset comprised only 3\% of the black ground-truth positive individuals marked as negative by the model.}
\label{fig:SSL_eq_app2}
\end{figure*}

\begin{figure*}
\begin{subfigure}{\textwidth}
\centering
\begin{tabular}{lr}
Feature      & Mean Sign \\ \hline
LSAT & $1.0$      \\
GPA        & $1.0$    \\
Race     & $*$      \\
Gender     & $0$       \\
\end{tabular}
\quad
\begin{tabular}{lr}
Feature      & Mean Diff \\ \hline
LSAT & $-1.43 $      \\
GPA     & $ -0.79 $      \\
Race     	 & $ * $     \\
Gender  & $ 0 $            \\
\end{tabular}
\subcaption{Transparency report created by using a GAN approximation of optimal transport to generate alternative inputs.}
\end{subfigure}
\begin{subfigure}{\textwidth}
\centering
\begin{tabular}{lr}
Feature      & Mean Sign \\ \hline
LSAT & $1.0$      \\
GPA        & $1.0$    \\
Race     & $*$      \\
Gender     & $0$       \\
\end{tabular}
\quad
\begin{tabular}{lr}
Feature      & Mean Diff \\ \hline
LSAT & $-1.88 $      \\
GPA     & $ -0.96 $      \\
Race     	 & $ * $     \\
Gender  & $ 0 $            \\
\end{tabular}
\caption{Transparency report created by using a Kusner et al.~\cite{kusner2017counterfactual} causal model to generate counterfactual inputs.}
\end{subfigure}
\caption{Transparency reports for the experiment in Section~\ref{sec:counterfactual}. In Section~\ref{sec:counterfactual} we have shown that the demographics of the flipsets largely agree, and here we see that transparency reports agree as well. We do not report the changes to race, since by definition of a flipset, all individuals change race in the flipset. The transparency report also points to reliance on LSAT scores.}
\label{fig:counter_trans}
\end{figure*}

\end{document}